\documentclass{article}
\usepackage{arxiv}
\usepackage{authblk}

\usepackage[utf8]{inputenc} 
\usepackage[T1]{fontenc}    
\usepackage{hyperref}       
\usepackage{url}            
\usepackage{booktabs}       
\usepackage{amsfonts}       
\usepackage{nicefrac}       
\usepackage{microtype}      
\usepackage{lipsum}
\usepackage{graphicx}
\newenvironment{proof}{\paragraph{Proof:}}{\hfill$\square$}

\usepackage{bm}
\usepackage{amsmath}
\usepackage{dsfont}
\usepackage[noend,ruled]{algorithm2e}
\usepackage{hyperref}
\usepackage{graphicx}

\newtheorem{definition}{Definition}
\newtheorem{lemma}{Lemma}
\newtheorem{proposition}{Proposition}
\newtheorem{theorem}{Theorem}
\newtheorem{corollary}{Corollary}
\newtheorem{exmp}{Example}

\makeatletter
\allowdisplaybreaks
\newcounter{protocol}

\newcommand{\CalA}{\mathcal{A}}

\newcommand{\CalD}{\mathcal{D}}
\newcommand{\CalI}{\mathcal{I}}

\newcommand{\Real}{\mathbb{R}}
\newcommand{\bbE}{\mathbb{E}}
\newcommand{\bbP}{\mathbb{P}}

\newcommand{\bmx}{\bm{x}}

\newcommand{\bmz}{\bm{z}}

\newcommand{\bmg}{\bm{g}}

\newcommand{\stoppingtime}{T_{\text{stop}}}
\newcommand{\bmmu}{\bm{\mu}}

\newcommand{\bmxi}{\bm{\xi}}

\newcommand{\DeltaPAC}{(\delta,\epsilon)}

\title{Multi-thresholding Good Arm Identification with Bandit Feedback}

\author[1,2]{Xuanke Jiang\thanks{\texttt{3ie22005s@s.kyushu-u.ac.jp}}}
\author[2]{Sherief Hashima\thanks{\texttt{Sherief.hashima@riken.jp}}}
\author[1,2]{Kohei Hatano\thanks{\texttt{kohei.hatano@riken.jp}}}
\author[1]{Eiji Takimoto\thanks{\texttt{eiji@inf.kyushu-u.ac.jp}}}
\affil[1]{Informatics Dept, Faculty of Information Science and Electrical Engineering, Kyushu University, Japan}
\affil[2]{Computational Learning Theory Team, RIKEN-AIP, Fukuoka, 819-0395, Japan}

\begin{document}
\date{ }

\maketitle
\begin{abstract}
We consider a good arm identification problem in a stochastic bandit
setting with multi-objectives, where each arm $i \in [K]$ is associated
with a distribution $D_i$ defined over $R^M$.
For each round $t$, the player pulls an arm $i_t$ and receives an $M$-dimensional reward vector sampled according to $D_{i_t}$.
The goal is to find, with high probability, an $\epsilon$-good arm whose expected reward vector is larger than
$\bm{\xi} - \epsilon \mathbf{1}$, where $\bm{\xi}$ is a predefined threshold vector, and the vector comparison is component-wise.
We propose the Multi-Thresholding UCB~(MultiTUCB) algorithm with a sample complexity bound.
Our bound matches the existing one in the special case where
$M=1$ and $\epsilon=0$.
The proposed algorithm demonstrates superior performance compared to baseline approaches across synthetic and real datasets.
\end{abstract}

\section{Introduction}
Single objective thresholding problem was proposed by~\cite{locatelli2016optimal,kano2019goodarm}, where the player/algorithm aims to find a subset of all arms $[K]$ denoted as the good arm set $[K]_{\text{good}}$. An arm is good if it has an average reward that surpasses a predefined performance threshold $\xi$ with a high probability.
This problem mainly adheres to two mainstream directions. 
The first one is known as the \textit{fixed budget setting}~\cite{locatelli2016optimal}, in which a fixed computational or sampling budget $T$ is given. This framework focuses on deriving a theoretical upper bound on the probability of errors associated with the selection of the arms. 
The second framework, termed the~\textit{fixed confidence setting}~\cite{kano2019goodarm}, aims to minimize the total sample complexity required to reliably identify the set of arms above the threshold while adhering to a predefined confidence level specified by an error tolerance parameter $\delta$. In this case, it's required to ensure that the selected set of arms is correct with a confidence level of at least $ 1-\delta$. This setting is particularly relevant in applications where achieving a desired level of statistical accuracy is critical and minimizing resource consumption remains a key consideration. 
Both frameworks offer valuable insights into pure exploration in decision-making under uncertainty by addressing these distinct but complementary goals. 

Concurrently, identifying arms under specific requirements with bandit feedback has garnered significant attention within topics on machine learning and sequential decision-making ~\cite{zhao2023revisitinggoodarm,kano2019goodarm}. 
However, existing works only considered a single threshold, while in this work, we extend it into multiple thresholds with the following two motivating examples:

\begin{exmp}[Machine Retention in Manufacturing]
    Consider a factory that operates multiple types of machines, each responsible for producing different components. The factory manager is tasked with optimizing operational costs without necessarily identifying the single most efficient machine. Hence, the manager implements a policy where any machine is retained if it cannot match thresholds across multiple issues. For instance, all remaining machines should have an energy consumption lower than $\xi_1$, maintenance cost lower than $\xi_2$, labor cost lower than $\xi_3$, etc. This approach ensures that all machines contributing positively to the overall profit are kept in operation, even if some are not the most efficient relative to others. 
\end{exmp}

\begin{exmp}[Internet of Things~(IoT) System Deployment]
    In large-scale IoT deployments, a network operator must choose one configuration from a pool of candidate sensor node designs or communication protocols, each exhibiting uncertain performance due to hardware variability, environmental noise, and dynamic workloads. Each candidate is evaluated based on multiple key performance metrics such as communication latency, energy consumption, and data reliability. Rather than selecting the optimal configuration, the objective is to identify any one setup that satisfies minimum operational thresholds across all metrics—for example, latency below 100 ms, packet loss under 1\%, and energy usage within a predefined budget. Due to resource constraints, it is essential to minimize the number of evaluations required for testing each configuration. 
\end{exmp}

Existing research on Multi-Objective Optimization~(MOO), e.g.,~\cite{pereira2022reviewMOO}, is not applicable to the above scenarios. 
Researchers on MOO typically aim to achieve Pareto optimality, e.g.,~\cite{jiang2023multiMOO}, in which improving one objective is impossible without negatively affecting another. 
On the basis of this, finding the Pareto front under the Multi-Armed Bandit~(MAB) setting for MOO has also gathered plenty of attention~\cite{kaufmann2016complexitybestarmidentification,kone2023adaptive}, which aims to find a subset of arms that achieves the Pareto optimality to form the front with optimal sample complexity. The Track-and-Stop framework~\cite{kaufmann2016complexitybestarmidentification, garivier2024sequential} has been implemented in this problem. However, these methods are not suitable for our setting since the goal is different, as shown in Figure~\ref{fig:ParetoFrontandThresholding}.

\begin{figure}
    \centering
    \includegraphics[width=0.50\linewidth]{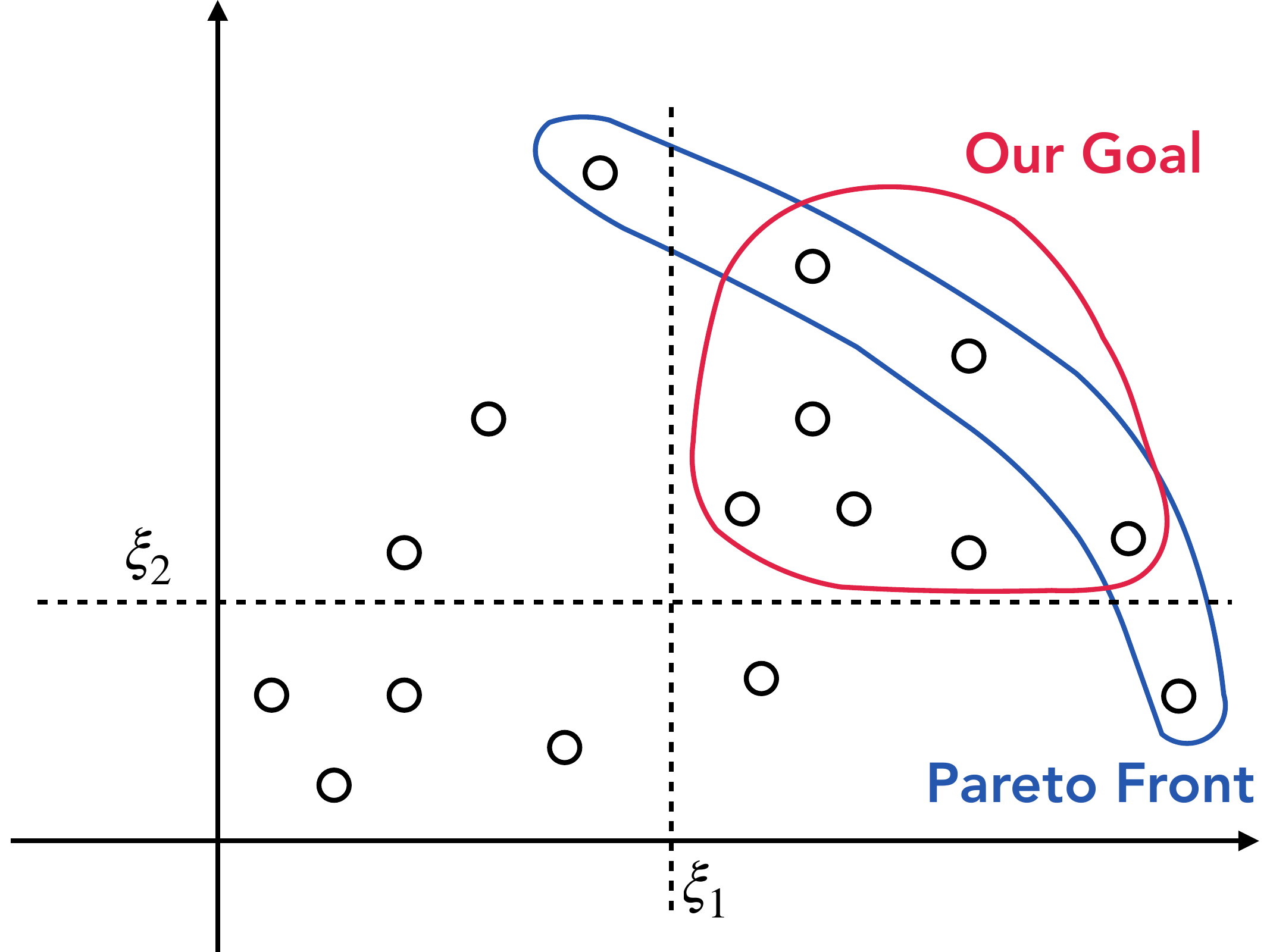}
    \caption{The comparison between the Pareto front and our goal, which is finding the arms whose expected average rewards exceed given thresholds.}
    \label{fig:ParetoFrontandThresholding}
\end{figure}

We integrate the difficulties of handling multiple objectives and restrictions, and thus we consider the case where multiple thresholds are simultaneously considered across all $K$ arms. Note that the concentration inequality for the single objective case, as $M=1$, cannot be applied directly to multi-objectives.
We address this challenge by establishing a new framework for good arm identification problems with multiple thresholds. At each round $t$, the player pulls one arm $i\in[K]$, in which $[K] = \{1,2,\ldots, K\}$, and receives a $M$ dimensional reward vector $\bmz_{t, i}$ sampled i.i.d. from a distribution $\CalD_i$. Concrete details are presented in section~\ref{sectionPreliminary}.
The player's goal is to identify a good arm, whose expected reward vector is larger than $\bm{\xi} - \epsilon\mathbf{1}$ for $\epsilon \in [0, 1]$. If no such arm exists, the player selects $\bot$ while minimizing sample complexity with high probability.



Our contributions can be summarized as follows:
\begin{enumerate}
    \item We introduce the Multi-Thresholding Good Arm Identification~(MTGAI) framework, which addresses the challenge of identifying arms that satisfy multiple thresholds simultaneously.
    \item  We propose the Multi-Thresholding UCB~(MultiTUCB) algorithm for the MTGAI problem and derive an upper bound on its sample complexity 
    that is consistent with existing results when $M=1$ and $\epsilon=0$~\cite{kano2019goodarm}. Our analysis significantly generalizes prior work by extending the theoretical guarantees from Bernoulli rewards to sub-Gaussian distributions and from zero accuracy~($\epsilon=0$) to a more practical setting with non-zero accuracy~($\epsilon > 0$), which introduces substantial technical complexity in the proofs. Moreover, we design a refined arm selection rule tailored to the multi-thresholding setting, enabling more efficient exploration under multiple constraints.
    \item  We establish an almost matching lower bound for the MTGAI problem when $\epsilon = 0$, showing the near-optimality of our approach. 
    \item 
Experimental results on both synthetic and real datasets demonstrate that our algorithm consistently outperforms three strong baselines adapted from prior thresholding bandit methods~\cite{kano2019goodarm,locatelli2016optimal,kalyanakrishnan2012pac}.
\end{enumerate}

    

\subsection{Other Related Work}
In addition to the papers discussed in the previous part, we discuss further related work 
 to our settings.

\textbf{Best Arm Identification~(BAI)}
The BAI is a pure exploration problem in which the agent tries to identify the best arm with the maximal mean reward within a limited sample~\cite{audibert2010BAI}. Unlike minimizing the regret for online learning, which involves the dilemma of exploration and exploitation, BAI focuses exclusively on exploration—gathering enough information to determine the arm that maximizes the mean reward. In addition, it has been extended to varying reward distributions, or constraints such as fixed confidence~\cite{garivier2016optimal,jourdan2023varepsilonfixconfidence}, making it a versatile tool in decision-making under uncertainty. 

\textbf{Multi-objective Optimization} 
As for MOO, achieving Pareto optimality is of essential importance for making in-time predictions, particularly in contexts where data explosion poses significant challenges. Various approaches have been developed for achieving Pareto optimality, including evolutionary methods~\cite{murata1995mogaMOO}, hypervolume scalarization~\cite{zhang2020randomMOO}, and multiple gradient-based methods~\cite{sener2018multiMOO}. This applies to various real applications, such as resource management~\cite{Applicationliu2024multiResourceAllocation} and load forecasting~\cite {Applicationxing2024novelForcasting}. 
Alternatively, researchers have also studied the problem from the perspective of regret minimization~\cite{lu2019multibanditMOO, cheu2018skyline}, which serves as a different goal. In parallel, researchers in economics have considered selecting one arm from the Pareto front that optimizes the Generalized Gini Index~\cite{busa2017multi}. 


However, all previous works do not solve our MTGAI problem, unlike our proposed MultiTUCB algorithm, which will be explained in Section~\ref{section:Multi-threshodlingUCBalgorithm}. 



\section{Preliminaries}
\label{sectionPreliminary}
This section presents basic tools and a formal problem formulation. First, we introduce the definition of $(\mu, \sigma)$-subgaussian random variable and its property.

\begin{definition}[$(\mu, \sigma)$-subgaussian]
    A random variable $z$ that takes values in $\Real$ is $(\mu, \sigma)$-subgaussian if $\mu = \bbE [\bmz]$ and for any $\lambda\in\Real$, 
    \[
    \bbE[\exp(\lambda(z - \mu))] \leq \exp(\lambda^2\sigma^2 /2 ).
    \]
\end{definition}

\begin{proposition}[See e.g.~\cite{lattimore2020bandit}]
\label{PropHoeffdingSubgaussian}
    Let $z_1, \ldots, z_t$ be i.i.d. \\
    $(\mu, \sigma)$-subgaussian random variables. Then for any $\epsilon\geq 0$, 
    \[
    \bbP\left(\hat\mu_t \geq \mu +\epsilon \right) \leq \exp{\left(-\frac{t\epsilon^2}{2\sigma^2}\right)} \text{ and } \bbP\left(\hat\mu_t \leq \mu -\epsilon\right) \leq \exp{\left(-\frac{t\epsilon^2}{2\sigma^2}\right)}, 
    \]
    where $\hat\mu_t = \frac{1}{t}\sum_{s=1}^t z_{s}$.
\end{proposition}

This paper investigates the protocol of MTGAI, which is defined between one player and Nature as the following protocol with $K$ as the number of arms and $M$ as the number of thresholds/objectives. 
They are given an accuracy parameter $\epsilon$, a confidence parameter $\delta$ and thresholds $\bmxi = \left(\xi^{(1)},\ldots, \xi^{(M)}\right) \in [0, 1]^{M}$.   
For each iteration $t = 1, 2, \ldots$, (i) the player chooses an arm $i_t\in[K]$ and (ii) Nature returns a reward vector ${\bmz}_{t, i_t}\in {[0,1]}^{M}$ sampled according to $\CalD_{i_t}$. Here, $\CalD_i$ is defined as the distribution over $\Real^{M}$ associated with arm $i$ such that for the random variable $\bmz = \left(z^{(1)}, \ldots, z^{(M)}\right)\sim\CalD_i$, $\bmz^{(m)}$ is $(\mu_i^{(m)}, \sigma)$-subgaussian for some $\bmmu_i = \left(\mu_i^{(1)}, \ldots, \mu_i^{(M)}\right)\in {[0,1]}^M$ and $\sigma\in\Real_{+}$ for each $m\in[M]$.
The player decides when to stop and outputs an arm $\hat {\imath}\in [K]\cup \{\perp\}$. The stopping time is considered as sample complexity. 
An arm $i$ is $\epsilon$-good if $\bmmu_i\geq \bmxi - \epsilon\bm{1}$, where the inequality is component-wise. In particular, an arm $i$ is good if it is $0$-good. The goal of the player is to find an arm $\hat\imath$ with $\bbP\left(\hat\imath\neq \bot \wedge \bmmu_{\hat\imath} \geq \bmxi - \epsilon\bm{1} \right) \geq 1-\delta$ if one good arm exists and to output $\bot$ with $\bbP\left(\hat\imath = \bot \right) \geq 1-\delta$ if no $\epsilon$-good arm exists. An algorithm that achieves this goal is said to be \emph{$\DeltaPAC$-successful}.
Notice that this protocol can be easily extended to finding all good arms by repeatedly running the algorithm for at most $K$ times. We summarize all other symbols in Table~\ref{TableNotation} for convenience.

\begin{table}[t]
\label{TableNotation}
\caption{Notation List}
    \centering
    \begin{tabular}{c l}
    \toprule
        $\bot$ & The bottom\\
        $T_i(t)$ & The number of times arm $i$ is chosen until round $t$ \\
        $\hat{a}$ & The output of Algorithm~\ref{alg:Multi thresholdUCB}\\
        $\stoppingtime$ & Stopping time of Algorithm~\ref{alg:Multi thresholdUCB}\\
        $\tilde{g}_{t,i}$ & $ = \hat{g}_{t,i} - \sqrt{\frac{2\sigma^2\ln{\left(KMT_i(t)\right)}}{T_i(t)}}$\\
        $\underline{g}_{t,i}$ & $ = \hat{g}_{t,i} - \alpha(T_{i}(t), \delta)$\\
        $\overline{g}_{t,i}$ & $ = \hat{g}_{t,i} + \alpha(T_{i}(t), \delta)$\\
        $i^{*}$ & $=\arg\min\limits_{i\in[K]} g_{i}$\\
    \bottomrule
    \end{tabular}
\end{table}

\section{Multi-thresholding UCB algorithm}
\label{section:Multi-threshodlingUCBalgorithm}
In this section, we propose the MultiTUCB algorithm to efficiently solve the proposed problem. 
First, we introduce a gap vector and its estimator to handle the multiple objectives and prepare for the algorithm. For each $i\in[K]$, let
\begin{align}\label{nablamu}
    g_{i} &= g(\bmmu_i) = \max\left\{\xi_1 - \mu_i^{(1)},\xi_2 - \mu_i^{(2)}, \ldots, \xi_M - \mu_i^{(M)}\right\}.
\end{align}
Note that for any $\epsilon$-good arm $i$, $g_i\leq \epsilon$.
The gap vector $\bmg$ is defined as $\bmg = (g_1, \ldots, g_K)$. 
\begin{definition}(\textbf{$\alpha$-approx estimator})
\label{defiAlphaApproxEstimator}
    For a function $\alpha: \mathbb{N}\times (0,1)\rightarrow \Real_{+}$, we define
    $\alpha$-approx estimator of $\bmg$ is an oracle that satisfies the condition. 
    For the $t$-th call to the oracle, it receives input $i_t\in [K]$ and returns $\hat {\bmg}_t = (\hat{g}_{t, 1}, \ldots, \hat{g}_{t, K}) \in\Real^{K}$ such that
    \begin{align}
        \forall\delta\in(0,1),\bbP\left( \forall i\in[K],\forall s\in[t], \left| \hat{g}_{s,i} - g_i\right|\leq \alpha(T_{i}(s),\delta) \right) \geq 1-\delta, 
    \end{align}
    where $T_i(s) = \left|\{ u\in[s] | i_u = i\}\right|$. 
\end{definition}

We claim that $\bmg$ satisfies the following proposition.

\begin{proposition}
    \label{prop:muLipschitz}
    For any $i\in [K]$ and any 
    $\bmmu_1,\ldots,\bmmu_K$, $\hat{\bmmu}_1,\ldots,\hat{\bmmu}_K \in {[0,1]}^M$
    \[
    \left|g(\bmmu_i) - {g}(\hat{\bmmu}_i) \right| \leq \|\bmmu_i - \hat{\bmmu}_i\|_\infty.
    \]
\end{proposition}The proof is given in Appendix~\ref{AppendixProofofProp2}. 

Now, we show the existence of $\alpha$-approx estimator for some $\alpha$. The algorithm for the estimator is described in Algorithm~\ref{algo:ApproxEstimator}.

\SetAlgoNoLine
\begin{algorithm}[H]
 \label{algo:ApproxEstimator}
    \caption{Approx Estimator}
    \For{$i\in[K]$} 
    {\begin{align*}
        \hat{\bmmu}_{t,i} = \frac{\sum^t_{s=1} \mathds{1}[i_s = i] \bmz_{s,i}}{T_i(t)};
    \end{align*}}
    \textbf{Output}: $\hat{\bmg}_t = (g(\hat{\bmmu}_{t, 1}),\ldots, g(\hat{\bmmu}_{t, K}))$;
\end{algorithm}

Hence, we calculate the exact value of the $\alpha$-approx estimator under our setting.
\begin{proposition}\label{Prop:alphaApproxestimatorvalue}
     Algorithm~\ref{algo:ApproxEstimator} is an $\alpha$-approx estimator of $\bmg$ with 
     \[
     \alpha(\tau,\delta) = \sqrt{\frac{2\sigma^2\ln{\frac{\pi^2 KM\tau^2}{3\delta}}}{\tau}}.
     \] 
\end{proposition}
The proof is given in Appendix~\ref{AppendixalphaApproxestimatorvalue}.

Then, we propose the MultiTUCB algorithm to efficiently solve the proposed problem. The detail of the algorithm is given in Algorithm~\ref{alg:Multi thresholdUCB}.
It uses Algorithm~\ref{algo:ApproxEstimator} for the $\alpha$-approx estimator as shown in Proposition~\ref{Prop:alphaApproxestimatorvalue} 
in the stopping criteria. 
We denote $T_{\text{stop}}$ as the stopping time and $\hat{a}$ as the output of Algorithm~\ref{alg:Multi thresholdUCB}.
\SetAlgoNoLine
\begin{algorithm}[ht]
    \caption{Multi-thresholding UCB algorithm~(MultiTUCB)}\label{alg:Multi thresholdUCB}
    \textbf{Input}: Arm set $[K]$, thresholds $\{\xi_1,\ldots,\xi_M\}$, confidence parameter $\delta$, accuracy rate $\epsilon$;\\
    \textbf{Init}: Let $\CalA = [K]$;\\
    Pull each arm once;\\
    Compute $\hat \bmg_{K}$ by Algorithm~\ref{algo:ApproxEstimator};\\
    \For{$t\in[K + 1, \ldots, ]$}
    {Select $i_t \in \arg\min\limits_{i\in \CalA} \tilde{g}_{t-1, i}$ with $\tilde{g}_{t - 1,i} = \hat g_{t-1,i} - \sqrt{\frac{2\sigma^2\ln{\left(KM{T_{i}(t-1)}\right)}}{T_{i}(t-1)}} $;\\
    Receive feedback $\bmz_{t,i_t}$;\\
    Update $T_i(t)$;\\
    Compute $\hat \bmg_{t}$ by Algorithm~\ref{algo:ApproxEstimator};\\
    \If {$\underline{g}_{t, i_t} > 0$}
    {Delete $i_t$ from $\CalA$.}
    //\textbf{Condition 1}\\
    \If{$\overline{g}_{t, i_t} \leq \epsilon$}{Output $\hat a = i_t$, Stop.}
    //\textbf{Condition 2}\\
    \If{$\CalA = \emptyset$}
    {Output $\hat a = \bot$, Stop.}
    }
\end{algorithm}
\subsection{Upper Bounds with Expectation}
\label{SectionUpperBoundswithExpectation}
First, we guarantee the correctness of our proposed algorithm with the following theorem.
\begin{theorem}
\label{theoremDeltaPACmultiTUCB}
    Algorithm~\ref{alg:Multi thresholdUCB} is $\DeltaPAC$-successful.
\end{theorem}
Details are deferred to Appendix~\ref{AppendixProoftheoremDeltaPACmultiTUCB}.


\begin{definition}
\label{definitiontiepsilon0}
    Let $t_i:\mathbb{R}_+ \rightarrow \mathbb{R}_+$ be a function defined as follows.
    \begin{align*}
        &\text{For any good arm i,}\\
        &\quad t_i(\epsilon_0) = \max\left\{ \frac{4\sigma^2}{{(\epsilon - g_i - \epsilon_0)}^2}\ln{\left(\frac{8\sqrt{3} \sigma^2\pi KM/\delta}{3{(\epsilon - g_i - \epsilon_0)}^2}\ln{\frac{4\sqrt{3}\pi\sigma^2}{3{(\epsilon - g_i - \epsilon_0)}^2}}\right)}, 0\right\}.\\
         &\text{For any non-$\epsilon$-good arm i,}\\
         &\quad t_i(\epsilon_0) = \max\left\{ \frac{4\sigma^2}{{(g_i - \epsilon -  \epsilon_0)}^2}\ln{\left(\frac{8\sqrt{3} \sigma^2\pi KM/\delta}{3{(g_i -\epsilon -  \epsilon_0)}^2}\ln{\frac{4\sqrt{3}\pi\sigma^2}{3{(g_i - \epsilon - \epsilon_0)}^2}}\right)}, 0\right\}.
    \end{align*}
\end{definition}

\begin{theorem}
\label{theoremUpperboundExpectationGoodArm}
If there exists a good arm, letting $i^{*} = \arg\min\limits_{i\in[K]} g_i$, Algorithm~\ref{alg:Multi thresholdUCB} achieves 
\begin{align}
    \bbE[\stoppingtime] &\leq  t_{i^*}(\epsilon_0)
    +\sum_{i\neq i^*}\frac{8\sigma^2\ln{\left(KM\max\limits_{i\in[K]} \lfloor t_i(\epsilon_0)\rfloor\right)}}{{\left(g_i - g_{i^*} +\epsilon_0\right)}^2}\nonumber \\
    &\quad +\frac{2(K+1)M\sigma^2}{{\epsilon_0}^2} + \frac{K^3M}{2{\epsilon_0}^2}e^{4{\epsilon_0}^2}, \nonumber
\end{align}
where $\epsilon_0$ be any number s.t. $0 < \epsilon_0 <  \epsilon - g_i$ for any good arm $i$ and be any number s.t. $0 < \epsilon_0 <  g_i - \epsilon$ with any non-$\epsilon$-good arm $i$.
\end{theorem}

\textbf{Proof sketch of Theorem~\ref{theoremUpperboundExpectationGoodArm}:}
The stopping time with an output $\hat a \in[K]$ can be divided as follows: (i) the number of rounds at which the algorithm chooses $\hat{a}$, (ii) the number of rounds at which the algorithm chooses another $\epsilon$-good arm, and (iii) the number of rounds at which the algorithm chooses any non-$\epsilon$-good arms. Bounding the expectation of each term and combining them lead to the final result. Appendix~\ref{AppendixProofTheoremUBExpectationGoodArm} details the full proof.

\begin{theorem}
\label{theoremUpperboundExpectationNogoodArm}
If there is no $\epsilon$-good arm, Algorithm~\ref{alg:Multi thresholdUCB} achieves 
\label{TheoremNoGoodArm}
    \begin{align}
      \bbE[\stoppingtime] & \leq \sum\limits_{i\in[K]} t_i(\epsilon_0) + \frac{KM\sigma^2}{\epsilon_0^2}\nonumber, 
    \end{align}
    while $\epsilon_0$ be any number such that $0 < \epsilon_0 <  g_i - \epsilon$ for any arm $i\in[K]$.
\end{theorem}

\textbf{Proof sketch of Theorem~\ref{theoremUpperboundExpectationNogoodArm}:}
If no good arm exists and the algorithm continues at round $t$, Condition $2$ is not satisfied, and thus the active set is not empty. This means some arm $i_t$ is chosen at round $t$ with $\hat{g}_{t, i_t} \leq \alpha(T_i(t), \delta)$. Bounding the number of such rounds over all arms leads to the result. The detailed proof of Theorem~\ref{theoremUpperboundExpectationNogoodArm} is postponed to Appendix~\ref{AppendixUpperBoundExpectationNogoodArm}. 

We present a more straightforward result based on the above theorems in the following corollary.
\begin{corollary}
    \label{corollaryExpectationUpperbound}
    Algorithm~\ref{alg:Multi thresholdUCB} achieves the following:
    \begin{enumerate}
        \item If there exists a good arm and letting $i^{*} = \arg\min\limits_{i\in[K]} g_i$,  
        \begin{align}
        \limsup\limits_{\delta\rightarrow 0} \frac{\bbE[\stoppingtime]}{\ln{(1/\delta)}}&\leq\frac{4\sigma^2 }{{\left(\epsilon - g_{i^*} -\epsilon_0\right)}^2}\nonumber,
        \end{align}
        where $\epsilon_0$ be any number such that $0 < \epsilon_0 <  \epsilon - g_i$ for any good arm $i$.
    \item If there is no $\epsilon$-good arm, 
    \begin{align}
        \limsup\limits_{\delta\rightarrow 0}\frac{\bbE[\stoppingtime]}{\ln{(1/\delta)}} &\leq \sum\limits_{i\in[K]} \frac{4{\sigma}^2}{{\left(\epsilon - g_{i} - \epsilon_0\right)}^2}\nonumber,
    \end{align}
    where $\epsilon_0$ be any number such that $0 < \epsilon_0 <  g_i - \epsilon$ with any arm $i\in[K]$.
    \end{enumerate}
\end{corollary}

\section{Lower Bound}
In the following part, we give a lower bound for the expectation of $T_{\text{stop}}$.

\begin{definition}[Binary relative entropy]
\label{defiBinaryRelativeentropy}
    For $x,y\in \mathbb{R}_+$, the binary relative entropy is defined as 
    \begin{align*}
        d(x,y)= x\ln(x/y) + (1-x)\ln{(1-x)/(1-y)})
    \end{align*}
    with conventions $d(0,0) = d(1,1) = 0$. 
\end{definition}
\begin{theorem}
    \label{ExpectationLowerbound}
    There exists a bandit model s.t. the following holds for any $(\delta, 0)$-successful algorithm with stopping time $T_{\text{stop}}$:
    \begin{enumerate}
        \item If good arm $i$ exists, 
    \begin{align*}
        \bbE[T_{\text{stop}}] \geq \frac{1}{\max\limits_{i\in[K]_{\text{good}}}\min\limits_{m\in[M]} d\left({\mu^{(m)}_{i}, \xi_m}\right)}\ln{\frac{1}{2\delta}} - \frac{\delta}{\max\limits_{i\in[K]_{\text{good}}}\min\limits_{m\in[M]} d\left({\mu^{(m)}_{i}, \xi_m}\right)},
    \end{align*}
    where $[K]_{\text{good}} \subseteq [K]$ is the set of all good arms.
    \item  When no good arm exists, 
     \begin{align*}
        \bbE[T_{\text{stop}}] \geq \frac{1}{\max\limits_{i\in[K]}\min\limits_{m\in[M]} d\left({\mu^{(m)}_{i}, \xi_m}\right)}\ln{\frac{1}{2\delta}} - \frac{\delta}{\max\limits_{i\in[K]}\min\limits_{m\in[M]} d\left({\mu^{(m)}_{i}, \xi_m}\right)}.
    \end{align*}
    \end{enumerate}
\end{theorem}
Detailed proof is in Appendix~\ref{AppendixLowerBoundExpectation}. 

\paragraph{Discussion on the tightness of the bounds}By Pinsker's inequality and~\cite{SasonIEEEITW2015FALL}, we obtain 
\begin{align*}
    \frac{2g_{i^*}^2}{a_{\bmxi}} \geq \max\limits_{i\in[K]_{\text{good}}}\min\limits_{m\in[M]} d\left({\mu^{(m)}_{i}, \xi_m}\right)  \geq 2g_{i^*}^2, 
\end{align*}
where $a_{\bmxi} := \min\limits_{m\in[M]}\min\{ \xi_m, 1-\xi_m \}$
Then we have
\[
\lim\limits_{\delta\rightarrow 0} \frac{\bbE[T_{\text{stop}}]}{\ln{1/\delta}} \geq \frac{2g_{i^*}^2}{a_{\bmxi}}. 
\]
When every element of $\bmxi$ is not close to $0$ or $1$, $a_{\bmxi}$ is large enough to be considered as a constant.
On the other hand, when feedback follows Bernoulli distributions, which is $\frac{1}{4}$-sub-Gaussian, $\epsilon = 0$ and a good arm exists, we have the following result for the upper bound
\begin{align*}
\lim\limits_{\delta\rightarrow 0} \frac{\bbE[T_{\text{stop}}]}{\ln{1/\delta}}& \leq \frac{1}{g_{i^*}^2}.
\end{align*}
Thus, we see that our upper bound for $(\delta,0)-$successful algorithm almost matches the lower bound when a good arm exists.

\section{Experiments}
\label{section:experiments}
In this section, we compare our proposed MultiTUCB with three benchmarks, which are Anytime
Parameter-free Thresholding algorithm for MTGAI~(MultiAPT)~\cite{locatelli2016optimal,jourdan2023varepsilonfixconfidence}, 
Lower and Upper
Confidence Bounds algorithm for MTGAI~(MultiLUCB)~\cite{kalyanakrishnan2012pac}
, Hybrid algorithm for the Dilemma of Confidence for MTGAI~(MultiHDoC)~\cite{kano2019goodarm} in terms of stopping time. 
All the algorithms are presented in Algorithm~\ref{alg:APTandHDoC}. As for MultiLUCB, we adopt a lower confidence bound for $g_i$, i.e.,  $\hat{g}_{i} - \alpha(T_i(t), \delta)$ as the selection rule. 
Note that Algorithm~\ref{alg:APTandHDoC} is guaranteed to solve the MTGAI problem for any benchmark algorithms and ours when it terminates by Theorem~\ref{theoremDeltaPACmultiTUCB}. However, unlike ours, the benchmark algorithms are designed for the single thresholding problem only and they do not have sample complexity bounds for the multi-thresholding one. All the codes can be obtained at~\href{https://github.com/2015211217/MultiThresholdBandit.git}{Github}.

\begin{algorithm}[!ht]

    \caption{MultiHDoC/MultiLUCB/MultiAPT-G}
    \label{alg:APTandHDoC}
    \textbf{Input}: Arm set $[K]$, thresholds $\{\xi_1,\ldots,\xi_M\}$, confidence parameter $\delta$, accuracy parameter $\epsilon$;\\
    \textbf{Init}: Let $\CalA_{\text{HDoC}} = [K]/\CalA_{\text{LUCB}} = [K]/\CalA_{\text{APT}} = [K]$; \\
    Pull each arm once;\\
     Compute $\hat \bmg_{K}$ by Algorithm~\ref{algo:ApproxEstimator};\\
    \For{$t\in[K + 1, \ldots, ]$}
    {\textbf{MultiHDoC}: Pull arm
    $i_t = \arg\min\limits_{i\in \CalA_{\text{HDoC}}} \hat{g}_{t, i} - \sqrt{\frac{\ln{t}}{2T_i(t)}}$.\\
     \textbf{MultiLUCB}: Pull arm 
        $i_t = \arg\min\limits_{i\in\CalA_{\text{LUCB}}} \hat{g}_{t, i} - \alpha(T_i(t), \delta)$.\\
    \textbf{MultiAPT}: Pull arm $i_t = \arg\min\limits_{i\in\CalA_{\text{APT}}} \sqrt{T_i(t)}\left| \hat{g}_{t,i} - \epsilon\right|.$\\
    Receive feedback $\bmz_{t,i_t}$;\\
    Update $T_i(t)$;\\
    Compute $\hat \bmg_{t}$ by Algorithm~\ref{algo:ApproxEstimator};\\
    \If {$\underline{g}_{t, i_t} > 0$}
    {Delete $i_t$ from $\CalA_{\text{HDoC}}/\CalA_{\text{LUCB}}/\CalA_{\text{APT}}$.}
    \If{$\overline{g}_{t,i_t} \leq \epsilon$}{
        \text{ }\\
        Output $\hat a = i_t$ as a good arm.\\
        Stop.\\
    }
    \If{$\CalA_{\text{HDoC}}/\CalA_{\text{LUCB}}/\CalA_{\text{APT}}= \emptyset$}
    {Output $\hat a = \bot$, Stop.}
    }
\end{algorithm}
\subsection{Synthetic Data}
\label{ExperimentsSyntheticData}
We consider the following $4$ different environments with $K = 10, M = 4$. 
\begin{enumerate}
    \item Gaussian distribution with means $\mu^{(1)}_{1:3} = 0.1$, $\mu^{(1)}_{4} = 0.35$, $\mu^{(1)}_{5} = 0.45$, $\mu^{(1)}_{6} = 0.55$, $\mu^{(1)}_{7} = 0.65$ and $\mu^{(1)}_{8:10} = 0.2$.
    \item Gaussian distributions with means  $\mu^{(2)}_{1:4} = 0.4 - 0.2^{1:4}, 
    \mu^{(2)}_{5} = 0.45, \mu^{(2)}_{6} = 0.55, \\
    \mu^{(2)}_{7:10} = 0.6 + {0.1}^{5-(1:4)}$. 
    \item Gaussian distributions with increasing means $\mu^{(3)}_{1:4} = (1 : 4) \cdot 0.05$, $\mu^{(3)}_{5}= 0.45$, $\mu^{(3)}_{6}= 0.55$ and $\mu^{(3)}_{7:10}= 0.65+(0 : 3)\cdot 0.05$. 
    \item Gaussian distributions with grouped means $\mu^{(4)}_{1:4} = 0.4$, $\mu^{(4)}_{5:8} = 0.5$, $\mu^{(6)}_{9:10} = 0.6$.
\end{enumerate}
We set thresholds $\{\xi_1,\xi_2,\xi_3,\xi_4\} = \{0.6, 0.5, 0.6, 0.5\}$ and $\sigma$ is chosen as $1.2$ while all distributions share the same variance for convenience. All the baselines in Algorithm~\ref{alg:APTandHDoC} share the same stopping criteria as Algorithm~\ref{alg:Multi thresholdUCB}. 
We set a time limit of $200,000$ arm-pulls, and passing the limitation is counted as an error. All results are averaged over $5000$ repetitions. 

\begin{figure}[!htbp]
\centering
\includegraphics[width=0.6\linewidth]{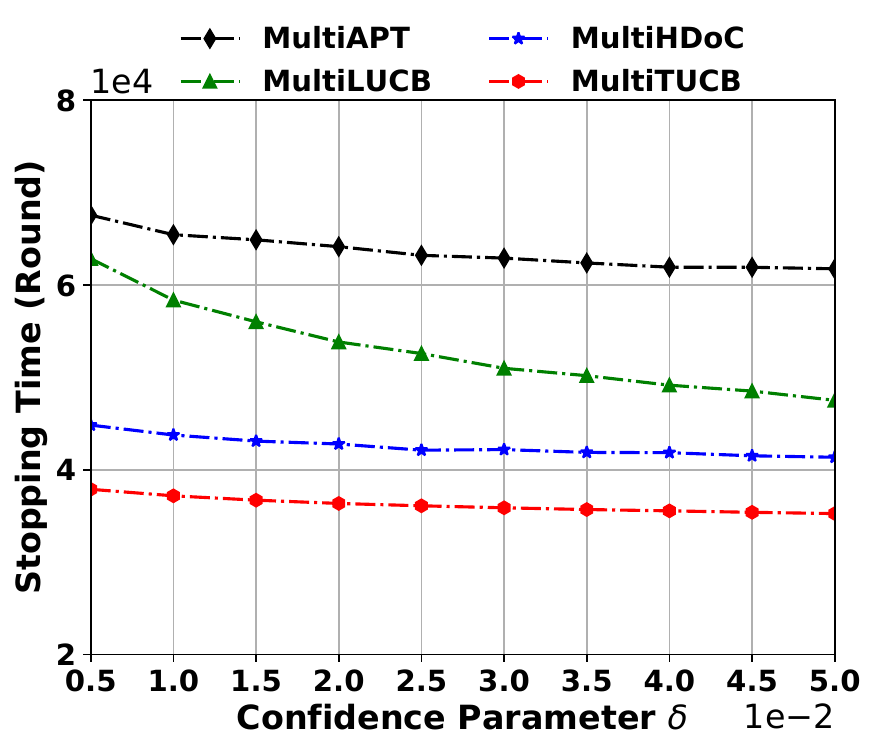}
\caption{Stopping times w.r.t. $\delta$ with synthetic data.}
\label{fig:SyntheticDelta}
\end{figure}

\begin{figure}[!htbp]
    \centering
    \includegraphics[width=0.6\linewidth]{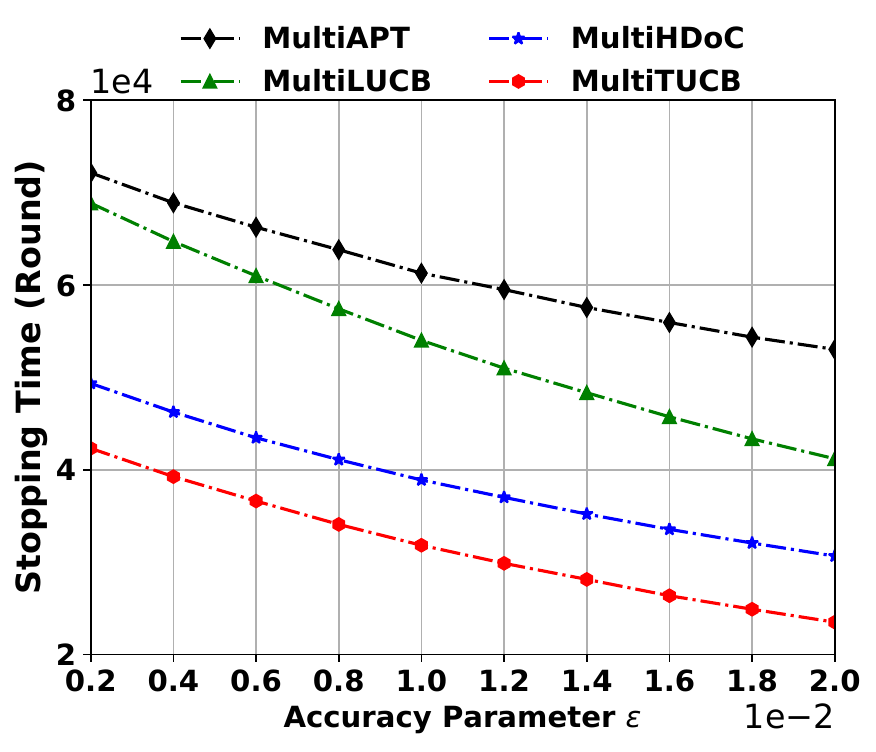}
    \caption{Stopping times w.r.t. $\epsilon$ with synthetic data.}
    \label{fig:SyntheticEpsilon}
\end{figure}

Figure~\ref{fig:SyntheticDelta}. exhibits the average stopping time against different values of $\delta$ from $0.005$ to $0.05$ for the compared schemes at $\epsilon = 0.005$. MultiTUCB outperforms other compared schemes since it is theoretically guaranteed. As $\delta$ increases, the difficulty of exiting the algorithm decreases, which may result in a shorter stopping time. 
Additionally, the performance curve of MultiLUCB is rugged, which indicates that it is more sensitive to the change of $\delta$ and needs more rounds to be stabilized. The others are not very sensitive to the change of $\delta$, but we observe that the number of required rounds decreases as the confidence parameter increases.

Figure~\ref{fig:SyntheticEpsilon} previews the stopping rounds of the compared schemes with the increasing $\epsilon = \{0.002, 0.004, \ldots, 0.02\}$ and $\delta = 0.005$. 
Here, we notice that MultiHDoC performs close to MultiTUCB since it's designed for finding a good arm set with a single threshold. It also indicates our modifications to the baseline algorithm are valid.
The stopping time of all algorithms decreases along with the relaxed stopping condition. Our proposed algorithm outperforms others not only in terms of straightforward stopping time but also in stability and less turbulence. 
We generalize the detailed data in Table~\ref{tab:experimentStandardDeviDeltaSynthetic} and
~\ref{tab:experimentStandardDeviEpsilonSynthetic} for comparison.

\begin{table}[htbp]
    \centering
    \caption{Standard Deviation of Stopping Time w.r.t. $\delta$ with Synthetic Data}
    \label{tab:experimentStandardDeviDeltaSynthetic}
    \begin{tabular}{c|c|c|c|c}
    \toprule
    $\delta$ & MultiAPT & MultiHDoC & MultiLUCB & \textbf{MultiTUCB} \\
    \midrule
    $0.005$ & $9408.58$ & $14589.53$ & $9640.70 $ & \bm{$7465.55 $} \\
    $0.010$ & $ 9197.93 $ & $ 14374.44 $ & $9464.46$ & \bm{$7393.21 $} \\
    $0.015$ & $9359.25 $ & $14167.67 $ & $ 9371.94 $ & \bm{$7354.64$} \\
    $0.020$ & $8861.37 $ & $13947.92 $ & $8785.98 $ & \bm{$ 7382.50 $} \\
    $0.025$ & $9154.26 $ & $13569.75 $ & $8885.88 $ & \bm{$ 7313.47 $} \\
    $0.030$ & $9020.14 $ & $  13688.36 $ & $8389.23 $ & \bm{$7302.48$} \\
    $0.035$ & $8967.06 $ & $13513.67 $ & $8278.27$ & \bm{$ 7233.80$} \\
    $0.040$ & $8909.38 $ & $13612.20$ & $ 8022.23 $ & \bm{$ 7225.78 $} \\
    $0.045$ & $9161.23 $ & $ 13393.54 $ & $7816.71 $ & \bm{$7237.46$} \\
    $0.050$ & $8914.17$ & $13429.67$ & $8034.80$ & \bm{$ 7213.29$} \\
    \bottomrule
    \end{tabular}
\end{table}
\begin{table}[htbp]
    \centering
    \caption{Standard Deviation of Stopping Time w.r.t. $\epsilon$ with Synthetic Data}
    \label{tab:experimentStandardDeviEpsilonSynthetic}
    \begin{tabular}{c|c|c|c|c}
    \toprule
    $\epsilon$ & MultiAPT & MultiHDoC & MultiLUCB & \textbf{MultiTUCB} \\
    \midrule
    $0.002$ & $10213.43 $ & $15203.78 $ & $10236.82$ & \bm{$8116.02$} \\
    $0.004$ & $ 9704.71 $  & $14821.15$ & $9840.27$ & \bm{$ 7615.36 $} \\
    $0.006$ & $ 9035.50 $  & $ 14501.22 $ & $ 9520.23  $  & \bm{$7284.30$} \\
    $0.008$ & $8803.26 $  & $14413.35 $ & $9018.05$  & \bm{$ 6953.02 $} \\
    $0.010$ & $8610.67 $  & $14300.52 $ & $ 8436.28  $  & \bm{$6672.03 $} \\
    $0.012$ & $ 8188.97$  & $14164.40 $ & $7995.11  $  & \bm{$6369.55$} \\
    $0.014$ & $ 8227.20$  & $14083.09$ & $7582.54 $  & \bm{$ 6099.59  $} \\
    $0.016$ & $8126.56$  & $ 14008.76 $ & $7144.09 $  & \bm{$5886.65 $} \\
    $0.018$ & $8042.25$  & $13878.92 $ & $6765.63$  & \bm{$5679.36 $} \\
    $0.020$ & $7956.99$  & $13810.01$ & $  6422.68$  & \bm{$5404.42$} \\
    \bottomrule
    \end{tabular}
\end{table}

\subsection{Dose Confirmation for Cocktail Therapy}
\label{subsectionCocktailTherapy}
Cocktail therapy refers to a treatment approach that combines multiple drugs or therapeutic agents to enhance effectiveness and achieve synergistic benefits. By targeting different pathways simultaneously, cocktail therapy can improve treatment outcomes, reduce drug resistance, and minimize side effects compared to single-drug treatments. Our work provides insights into confirming the dose of each medicine in cocktail therapy. Current research shows that overweight and type 2 diabetes are related~\cite{ruze2023obesity}, and we give insight into determining the dose of a possible cocktail therapy on two medicines for a better treatment. Here, we regularize all data into the range $(0,1]$ for convenience, and the larger means indicate the better curative effect. 
\begin{enumerate}
\label{enumerateMedical}
    \item  (Dose finding of LY3437943 compared with Dulaglutide~($1.5$mg) on glycated haemoglo-bin change w.r.t. placebo~\cite{ShwetaLY3437943Medical1}): Five Gaussian distributions with mean $\mu_1^{(1)}$ = 0.36, $\mu_2^{(1)} = 0.59$, $\mu_3^{(1)} = 0.85$, $\mu_4^{(1)}= 0.95$, $\mu_{5}^{(1)} = 0.79$,  corresponds to dose with unit mg as $\{1.5, 3, 3/6, 3/6/9/12\}$~($3/6$ means prescribe $3$mg for the first half of treatments and $6$mg for the next half), and threshold $\xi_1 = 0.48, \sigma = 1.0$.
    \item (Dose finding of cagrilintide compared with Liraglutide~($3.0$mg) on bodyweight change w.r.t. placebo~\cite{LAU20212160Medical2}): Five Gaussian distributions with mean $\mu_1^{(2)} = 0.375$, $\mu_2^{(2)} = 0.475$, $\mu_3^{(2)} = 0.7625$, $\mu_4^{(2)} = 0.8375$, $\mu_5^{(2)} = 0.975$, corresponds to dose with unit mg as $\{0.3,0.6,1.2,2.4,4.5\}$ with $\sigma =1.0$ and threshold $\xi_2 = 0.75$.
\end{enumerate}

\begin{figure}[htbp]
\centering
\includegraphics[width=0.6\linewidth]{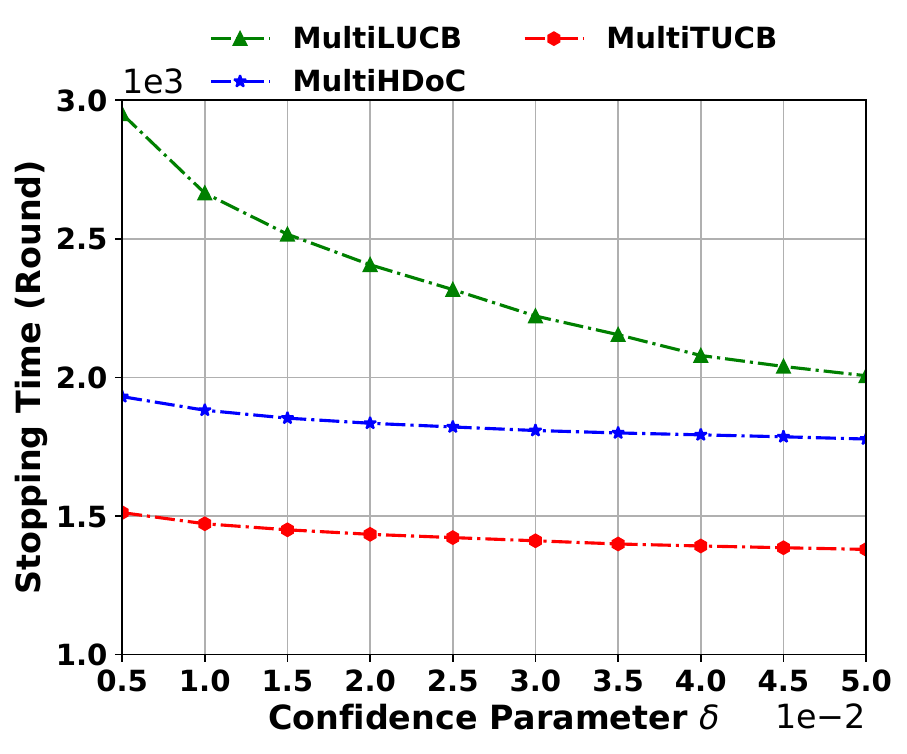}
\caption{Stopping times w.r.t. $\delta$ with medical data.}
\label{fig:DeltaMedicine}
\end{figure}

\begin{figure}[htbp]
\centering
\includegraphics[width=0.6\linewidth]{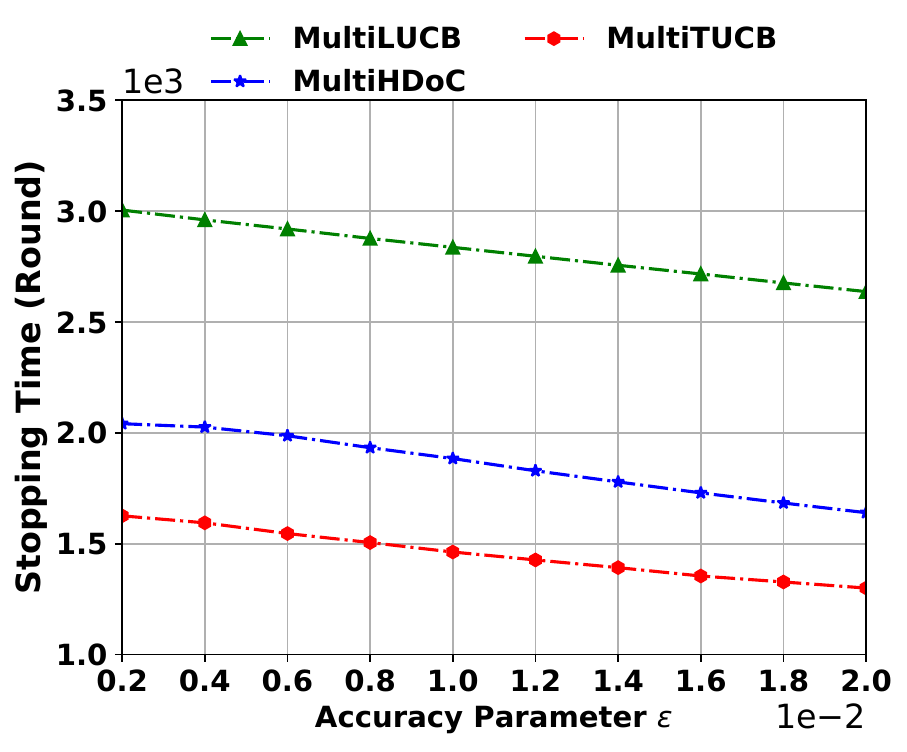}
\caption{Stopping times w.r.t. $\epsilon$ with medical data.}
\label{fig:EpsilonMedicine}
\end{figure}

We combine LY3437943~\cite{ShwetaLY3437943Medical1} and cagrilintide~\cite{LAU20212160Medical2} as shown above and put the results in Figure~\ref{fig:DeltaMedicine} and Figure~\ref{fig:EpsilonMedicine}. The algorithm MultiAPT has a large stopping time compared with others in this setting, and we omit it from the figures for clarity. The behavior of MultiHDoC, MultiLUCB, and MultiTUCB is similar to the one in Section~\ref{ExperimentsSyntheticData}.

Other complementary results are listed in Appendix~\ref{AppendixExperiments}.

\section{Conclusion}
This work introduced the multi-thresholding good arm identification framework in a stochastic bandit setting, where an arm is considered good if all components of its expected reward vector exceed given thresholds. We proposed the MultiTUCB algorithm, which achieves an upper bound on sample complexity scaling as $O(\ln{M})$ w.r.t. $M > 1$ in the dominant term. Meanwhile, we derive corresponding lower bounds, and we prove that it's an almost match when a good arm exists and $\epsilon = 0$. Experimental results demonstrated that MultiTUCB outperforms three baseline methods on synthetic and real-world datasets. 

As a direction for future work, it would be valuable to derive a better bound for the cases when no good arm exists or $\epsilon > 0$. Additionally, exploring more challenging settings, such as high-dimensional feedback, is another promising direction.

\bibliographystyle{unsrt}
\bibliography{reference}

\newpage
\appendix 
\section{Proof of Other Theoretical Results}
\subsection{Proof of Proposition~\ref{Prop:alphaApproxestimatorvalue}}
\label{AppendixalphaApproxestimatorvalue}
\begin{proof}
    For any $\delta\in(0,1)$, we have
    \begin{align*}
        &\bbP\left(\exists i\in[K], \exists s \in[t],  |\hat{g}_{s,i} - g_i| > \alpha(T_i(s), \delta) \right)\\
        &\quad\leq \bbP\left(\exists i\in[K], \exists s \in[t], \left\| \hat{\bmmu}_{s,i} -{\bmmu}_{i} \right\|_{\infty}> \alpha(T_i(s), \delta) \right)\quad\text{ (Proposition~\ref{prop:muLipschitz})}\\
        &\quad = \bbP\left(\exists i\in[K], \exists s \in[t], \exists m\in[M], \left|\hat{\mu}^{(m)}_{s,i} - \mu_{i}^{(m)}\right|  > \alpha(T_i(s), \delta)\right)\\
        &\quad\leq \sum\limits_{i\in[K]}\sum\limits_{s\in[t]}\sum\limits_{m\in[M]}\bbP\left(\left|\hat{\mu}^{(m)}_{s,i} - \mu_{i}^{(m)}\right| > \alpha(T_i(s), \delta)\right)\quad\text{(Union Bound)}\\
        &\quad\leq 2KM\sum\limits_{s\in[t]}\exp{\left(-\frac{{T_i(s)\alpha(T_i(s), \delta)}^2}{2\sigma^2}\right)}\quad\text{(Proposition~\ref{PropHoeffdingSubgaussian})}\\
        &\quad \leq 2KM\sum\limits_{\tau = 1}^\infty \exp{\left(-\frac{{\tau\alpha(\tau, \delta)}^2}{2\sigma^2}\right)}\\
        &\quad = \frac{6\delta}{\pi^2} \sum_{\tau = 1}^\infty \frac{1}{\tau^2}\quad(\text{By our choice of }\alpha)\\
        &\quad \leq \delta.
    \end{align*}
\end{proof}
\subsection{Proof of Proposition~\ref{prop:muLipschitz}}
\label{AppendixProofofProp2}
\begin{proof}
Fix $i\in[K]$ and $\bmx\in\Delta_K$ arbitrarily, and let 
\begin{align*}
    A_{M-1} &= \max\left\{\xi_1 - \mu_i^{(1)}, \ldots, \xi_{M-1} - \mu_i^{(M-1)}\right\} \\
    \hat A_{M-1} &= \max\left\{\xi_1 - \hat\mu_i^{(1)}, \ldots, \xi_{M-1} - \hat\mu_i^{(M-1)}\right\}.
\end{align*}
Then, by equation~\ref{nablamu} it suffices to show that 
\begin{equation}
    \label{equa:*}
    \left|A_M - \hat A_M \right|\leq \max\left\{\left|\mu_i^{(1)} - \hat \mu_i^{(1)}\right|,\ldots,\left|\mu_i^{(M)} - \hat \mu_i^{(M)}\right| \right\}.
\end{equation}
We prove this by induction on $M$, 
(Basis) For $M=1$:
\begin{align*}
    &\left|\left(\nabla_{\bmx}L(\bmx)\right)_{i} - (\nabla_{\bmx}L(\hat{\bmmu}_1,\ldots,\hat{\bmmu}_K,\bmx))_{i} \right|\\
    &\quad = \left|\max\left\{\xi_1 - \mu_i^{(1)}\right\} - \max\left\{\xi_1 - \hat\mu_i^{(1)}\right\}\right|\\
    &\quad = \left|\hat\mu_i^{(1)} - \mu_i^{(1)}\right|.
\end{align*}
(Induction Step) We assume equation~(\ref{equa:*}) holds for $M-1$, 
    \begin{align*}
        &\left|A_M - \hat A_M\right| \\
        &\quad = \left| \max\{A_{M-1}, \xi_M-\mu_i^{(M)}\} - \max\{\hat A_{M-1}, \xi_M - \hat\mu_i^{(M)}\} \right|\\
        &\quad = \bigg| \frac{1}{2}[A_{M-1} + \xi_M - \mu_i^{(M)} + |A_{M-1} - (\xi_M - \mu_i^{(M)})|]\\    
        &\qquad - \frac{1}{2}\big[\hat{A}_{M-1} + \xi_M - \hat\mu_i^{(M)} + |\hat{A}_{M-1} - (\xi_M - \hat \mu_i^{(M)})|\big] \bigg|\\
        &\quad \leq \left|\frac{1}{2}\left[A_{M-1} - \hat A_{M-1} + \hat\mu_i^{(M)} - \mu_i^{(M)}\right] + \frac{1}{2}\left| A_{M-1} - \hat{A}_{M-1} + \mu_i^{(M)} - \hat\mu_i^{(M)} \right|\right|\\
        &\quad  = \left| \max\{A_{M-1} - \hat A_{M-1}, \hat\mu_i^{(M)} - \mu_{i}^{(M)}\} \right|\\
        &\quad \leq \max\left\{\left|A_{M-1} - \hat A_{M-1}\right|, \left|\hat\mu_i^{(M)} - \mu_{i}^{(M)}\right|\right\}\\
        &\quad \leq \max\left\{|\mu_i^{(1)} - \hat{\mu}_i^{(1)}|, \ldots, |\mu_i^{(M)} - \hat{\mu}_i^{(M)}| \right\}, 
    \end{align*}
    and for any $a,b\in\Real$,
   \begin{align}
       \max\{a,b\}& = \frac{a+b}{2} + \frac{|a-b|}{2} \label{equa:ab1/2}\\
       |a| - |b| &\leq |a - b| \label{equa:aba-b}.
   \end{align}
    The second and third equations hold for~(\ref{equa:ab1/2}) while the first inequality stands over~(\ref{equa:aba-b}), and others are standard algebra.
\end{proof}

\subsection{Proof of Theorem~\ref{theoremDeltaPACmultiTUCB}}
\label{AppendixProoftheoremDeltaPACmultiTUCB}
In the following, we abuse the notations and denote $\hat{\mu}_{n,i}^{(m)}  = \sum\limits_{s\in{\mathcal{T}_i(t)}} z_{s, i}^{(m)} / n$, where $t = \min\{ s\mid T_i(s) = n \}$ and $\mathcal{T}_{i}(t) = \{s\in[t] \mid i_s = i\}$. Similarly, we write $\hat{g}_{n,i} = \max\{\xi_1 - \hat{\mu}_{n, i}^{(1)}, \ldots, \xi_M - \hat{\mu}_{n, i}^{(M)}\}$, $\overline{g}_{n, i} = \hat{g}_{n, i} + \alpha(n, \delta)$, and $ \underline{g}_{n,i} = \hat{g}_{n, i} - \alpha(n, \delta)$, respectively. Then, we give Proposition~\ref{AppendixModifiedProposition1} and Lemma~\ref{AppendixlemmaTheorem2} for preparation. 

\begin{proposition}
    \label{AppendixModifiedProposition1}
    For any $\epsilon \geq 0$, arm $i\in[K]$ and $n$, 
    \begin{align*}
        \bbP\left(\left| g_i - \hat{g}_{n, i} \right| \geq \epsilon \right)\leq 2M \exp\left( -\frac{n\epsilon^2}{2\sigma^2}\right).
    \end{align*}
\end{proposition}
\begin{proof}
    \begin{align*}
        \bbP\left(\left| g_i - \hat{g}_{t, i} \right| \geq \epsilon \right)&\leq \bbP\left(\left\|\hat{\bmmu}_{t, i} - \bmmu_i \right\|_{\infty} \geq \epsilon \right)\quad(\text{Proposition~\ref{prop:muLipschitz}})\\
        & = \bbP\left( \exists m\in[M],  \left| \hat{\mu}_{t, i}^{(m)} - \mu_i^{(m)} \right| \geq \epsilon \right)\\
        &\leq \sum\limits_{m\in[M]} \bbP\left( \left| \hat{\mu}_{t, i}^{(m)} - \mu_i^{(m)} \right| \geq \epsilon \right) \quad(\text{Union Bound}) \\
        & \leq 2M \exp\left( -\frac{n\epsilon^2}{2\sigma^2}\right)\quad(\text{Proposition~\ref{PropHoeffdingSubgaussian}}).
    \end{align*}    
\end{proof}

\begin{lemma}
    \label{AppendixlemmaTheorem2}
    For Algorithm~\ref{alg:Multi thresholdUCB}, we have
    \begin{align*}
        \bbP\left( \bigcup\limits_{n\in\mathbb{N}} \{ \underline{g}_{n,i} > 0\} \right) &\leq\frac{\delta}{K} , \quad\text{for any good arm }i, \text{ and}\\
        \bbP\left( \bigcup\limits_{n\in\mathbb{N}} \{ \overline{g}_{n,i} \leq \epsilon \} \right) &\leq \frac{\delta}{K}, \quad\text{for any non-$\epsilon$-good arm $i$}.
    \end{align*}
\end{lemma}
\begin{proof}
    As for any non-$\epsilon$-good arm $i\in[K]$, 
    \begin{align} \bbP\left(\bigcup\limits_{n\in\mathbb{N}} \{ \overline{g}_{n,i} \leq \epsilon\} \right)&\leq \sum\limits_{n\in\mathbb{N}} \bbP\left( \overline{g}_{n,i} \leq \epsilon \right)\quad(\text{Union bound})\nonumber\\
        &\leq \sum\limits_{n\in\mathbb{N}} \bbP\left(\overline{g}_{n,i} \leq {g}_{i}\right)\quad(\text{$g_i > \epsilon$ for any non-$\epsilon$-good arm})\nonumber\\
        &\leq \sum\limits_{n\in\mathbb{N}} 2M e^{-\frac{n{\left(\sqrt{\frac{2\sigma^2\ln{(\pi^2KMn^2/3\delta)}}{n}
        }\right)}^2}{2\sigma^2}} \quad(\text{Proposition~\ref{AppendixModifiedProposition1}})\nonumber\\
        & = \sum\limits_{n\in\mathbb{N}}\frac{6M\delta}{\pi^2KMn^2}\nonumber\\
        & \leq \frac{\delta}{K} \quad\left(\text{By }\sum\limits_{n\in\mathbb{N}}\frac{1}{n^2} \leq \frac{\pi^2}{6}\right).\label{MediateTheorem2}
    \end{align}    
    Next, consider any good arm $i\in[K]$, 
    \begin{align*}
        \bbP\left( \bigcup\limits_{n\in\mathbb{N}} \{ \underline{g}_{n,i} > 0\} \right) &\leq \sum_{n\in\mathbb{N}} \bbP\left( \underline{g}_{n,i} > 0 \right)\quad(\text{Union bound})\\
        &\leq \sum_{n\in\mathbb{N}} \bbP\left(\underline{g}_{n,i}> {g}_{n,i}\right)\quad(\text{$g_i \leq 0$ for any good arm})\\
        &\leq \frac{\delta}{K}  \quad(\text{Same arguments as~(\ref{MediateTheorem2})}).
    \end{align*}
\end{proof}

\begin{proof}[Theorem~\ref{theoremDeltaPACmultiTUCB}]
\label{AppendixProofofDELtaPAC}
    First, if no $\epsilon$-good arm exists, the failure probability is at most
    \begin{align}
        \bbP\left(\exists i\in[K], \bigcup\limits_{n\in\mathbb{N}} \{ \overline{g}_{n,i}\leq \epsilon\} \right) \leq \delta \label{AppendixFailureCase1}
    \end{align}
    by using the union bound and Lemma~\ref{AppendixlemmaTheorem2}.
    
    Similarly, if there exists a non-empty good arm set 
    $$
    [K]_{\text{good}} = \left\{i_1^{\text{good}}, \ldots, i_{|[K]_{\text{good}}|}^{\text{good}}\right\}, 
    $$
    then the failure probability is given as 
    \begin{align}
         &\bbP\left(\hat a = \bot \cup g_{\hat a} > \epsilon \right)\nonumber\\
         &\quad \leq\bbP\left(\hat a = \bot \right) + \bbP\left(g_{\hat a} > \epsilon \right)\quad(\text{Union bound}).\label{eqtheoremDeltaPACmultiTUCB1}
    \end{align}
    We give an upper bound for each term in~(\ref{eqtheoremDeltaPACmultiTUCB1}). First, 
    \begin{align}
        &\bbP\left(\hat a = \bot \right) \nonumber\\
        &\quad \leq \bbP\left(\forall i\in[K]_{\text{good}}, \bigcup\limits_{n\in\mathbb{N}} \{ \underline{g}_{n,i} > \epsilon\}\} \right)\nonumber\\
        &\quad \leq \bbP\left( \bigcup\limits_{n\in\mathbb{N}} \{ \underline{g}_{n,i} > \epsilon\} \text{ for a particular good arm $i$}\right)\nonumber\\
        &\quad\leq \frac{\delta}{K}\quad(\text{Lemma~\ref{AppendixlemmaTheorem2}}). \label{eqtheoremDeltaPACmultiTUCB2}
    \end{align}
    Then, for the second part, 
    \begin{align}
         &\bbP\left(g_{\hat a} > \epsilon \right)\nonumber\\
         &\quad \leq \bbP\left(\exists \text{non-$\epsilon$-good arm $i$ s.t. }\bigcup\limits_{n\in\mathbb{N}} \overline{g}_{i} > \epsilon \right) \nonumber\\
         &\quad \leq \frac{(K-1)\delta}{K}, \label{eqtheoremDeltaPACmultiTUCB3}
    \end{align}
    where the last inequality is obtained by using Lemma~\ref{AppendixlemmaTheorem2} and the fact that there are at most $K-1$ non-$\epsilon$-good arms since a good arm exists.
    Combining~(\ref{eqtheoremDeltaPACmultiTUCB1}),~(\ref{eqtheoremDeltaPACmultiTUCB2}) and~(\ref{eqtheoremDeltaPACmultiTUCB3}) leads to the result.
    
\end{proof}


\subsection{Proof of Theorem~\ref{theoremUpperboundExpectationGoodArm}}
\label{AppendixProofTheoremUBExpectationGoodArm}
\begin{lemma}
    \label{lemma2UpperBound}
    The following statements hold.
    \begin{enumerate}
        \item For any good arm $i$ with $\epsilon_0$ be any number such that $0 < \epsilon_0<\epsilon - g_i$ and any $n > t_i(\epsilon_0)$, 
        \begin{align}
        \bbP\left(\overline{g}_{n,i}>\epsilon\right) &\leq Me^{-\frac{n{\epsilon_{0}}^2}{2\sigma^2}}.\label{lemma2UpperBoundResult1}
        \end{align}
    \item For any non-$\epsilon$-good arm $i$ with $\epsilon_0$ be any number such that $0 < \epsilon_0<g_i - \epsilon$ and any $n > t_i(\epsilon_0)$,
    \begin{align}
        \bbP\left(\underline{g}_{n,i} \leq \epsilon \right) &\leq Me^{-\frac{n{\epsilon_{0}}^2}{2\sigma^2}}. \label{lemma2UpperBoundFailureCase}
    \end{align}
    \end{enumerate}
\end{lemma}
\begin{proof}
Here, we only focus on~(\ref{lemma2UpperBoundResult1}), the proof for the other half is similar. 
    By assumption, 
    \[
    n  > t_i(\epsilon_0) = \max\left\{ \frac{4\sigma^2}{{(\epsilon - g_i - \epsilon_0)}^2}\ln{\left(\frac{8\sqrt{3} \sigma^2\pi KM/\delta}{3{(\epsilon - g_i - \epsilon_0)}^2}\ln{\frac{4\sqrt{3}\pi\sigma^2}{3{(\epsilon - g_i - \epsilon_0)}^2}}\right)} , 0\right\}.
    \]
    First, we show
    \begin{align}
        \sqrt{\frac{2\sigma^2\ln{\frac{\pi^2 KMn^2}{3\delta}}}{n}}  < \epsilon - g_i - \epsilon_0\label{Mediate1}.
    \end{align}     
    Let $a = {(\epsilon - g_i - \epsilon_0)}^2$ for simplicity. Since the left hand side of~(\ref{Mediate1}) is monotone decreasing w.r.t. $n$, it's suffices to show that~(\ref{Mediate1}) is satisfied for $n = t_i(\epsilon_0) = \frac{4\sigma^2}{a}\ln{\frac{\pi^2 b\sqrt{KM/\delta}}{3a}}$ with
    \begin{align}
        b=\max\left\{ \frac{8\sqrt{3}\sigma^2}{\pi} \ln{\frac{4\sqrt{3}\pi\sigma^2\sqrt{KM/\delta}}{3a}}, \frac{3a}{\pi^2\sqrt{KM/\delta}}\right\}\label{Mediate4}.
    \end{align}
    For the case $b = \frac{3a}{\pi^2\sqrt{KM/\delta}}$, the inequality~(\ref{Mediate4}) holds trivially. Then, we consider the other part and let $A = {\frac{\pi^2\sqrt{KM/\delta}}{3a}}$, and $B = \pi /(4\sqrt{3}{\sigma^2})$. Then we have $\ln{\frac{\pi^2 b\sqrt{KM/\delta}}{3a}} = \ln{Ab}$ and $b = \frac{2}{B}\ln{\frac{A}{B}}$,
    \begin{align*}
        \ln{Ab} &= \ln{\frac{2A}{B}\ln{\frac{A}{B}}} \\
        &= \ln{\frac{A}{B}} + \ln{2\ln{\frac{A}{B}}} \\
        &\leq 2\ln{\frac{A}{B}} = Bb\quad(\ln{2x} < x), 
    \end{align*}
    

which indicates $\ln{\frac{\pi^2 b\sqrt{KM/\delta}}{3a}} < \pi b/4\sqrt{3}{\sigma^2}$.
\begin{align}
&\ln{\frac{\pi^2 b\sqrt{KM/\delta}}{3a}}  < \pi b/4\sqrt{3}{\sigma^2}\nonumber\\
&\quad\Leftrightarrow {\left(4\sigma^2\ln{\frac{\pi^2 b\sqrt{KM/\delta}}{3a}} \right)}^2 < \frac{\pi^2 {b}^2}{3}\nonumber \\     
&\quad \Leftrightarrow \frac{KM{\left( 4\sigma^2\ln{\frac{\pi^2 b\sqrt{KM/\delta}}{3a}} \right)}^2}{3\delta a^2} < \frac{{{\pi}^2b}^2KM}{9\delta a^2}\nonumber \\
&\quad\Leftrightarrow \ln{\frac{\pi^2KM{\left( \frac{4\sigma^2}{a}\ln{\frac{\pi^2 b\sqrt{KM/\delta}}{3a}}\right)}^2}{3\delta }} < 2\ln{\frac{\pi^2 b\sqrt{KM/\delta}}{3a}}\label{Mediate3}.
\end{align}
Then we have
\begin{align*}
 \sqrt{\frac{2\sigma^2\ln{\frac{\pi^2 KMn^2}{3\delta}}}{n}} & =\sqrt{\frac{2\sigma^2\ln{\frac{\pi^2KM{\left( \frac{4\sigma^2}{a}\ln{\frac{\pi^2 b\sqrt{KM/\delta}}{3a}} \right)}^2}{3\delta}}}{\frac{4\sigma^2}{a}\ln{\frac{\pi^2 b\sqrt{KM/\delta}}{3a}}}} \\
 &\leq \sqrt{\frac{4\sigma^2\ln{\frac{\pi^2 b\sqrt{KM/\delta}}{3a}} }{\frac{4\sigma^2\ln{\frac{\pi^2 b\sqrt{KM/\delta}}{3a}}}{a}}} \quad(\text{Inequality~(\ref{Mediate3})})\\
 &\leq \sqrt{a} = \epsilon - g_i - \epsilon_0, 
\end{align*}
which shows~(\ref{Mediate1}) holds.
    Next, 
    \begin{align}
        &\overline{g}_{n,i} > \epsilon\nonumber\\
        &\quad\Leftrightarrow \hat g_{n,i} >\epsilon - \sqrt{\frac{2\sigma^2\ln{\frac{\pi^2 KMn^2}{3\delta}}}{n}}\nonumber\\
        &\quad\Rightarrow \hat{g}_{n,i} \geq g_i + \epsilon_0 \quad(\text{Inequality~(\ref{Mediate1})})\nonumber\\
        &\quad\Rightarrow \exists j\in[M], \xi_j - \hat{\mu}_{n,i}^{(j)} \geq  \xi_j - {\mu}_{n,i}^{(j)} + \epsilon_0\nonumber.
    \end{align}
    Take the probability on both sides, 
     \begin{align}
        & \bbP\left(\overline{g}_{n,i} > \epsilon\right)\nonumber\\
        & \quad\leq \sum\limits_{j\in[M]} \bbP\left(\left\{ \hat{\mu}_{n,i}^{(j)} \leq \mu_{n,i}^{(j)} - \epsilon_0 \right\}\right)\nonumber\\
        & \quad\leq M e^{-\frac{n\epsilon_{0}^2}{2\sigma^2}} \quad(\text{Union bound and proposition~\ref{PropHoeffdingSubgaussian}})\nonumber.
    \end{align}
    
\end{proof}

\begin{lemma}
\label{lemma3UpperBound}
We have 
\begin{enumerate}
    \item  For any good arm $i$ with $\epsilon_0$ be any number such that $0 < \epsilon_0<\epsilon - g_i$, 
    \begin{align}
        \bbE\left[\sum_{n=1}^{\infty}\mathds{1}\left[\overline{g}_{n,i} > \epsilon \right]\right] &\leq t_i(\epsilon_0) + \frac{2M\sigma^2}{\epsilon_0^2}.\label{111111}
    \end{align}
    \item For any non-$\epsilon$-good arm $i$ with $\epsilon_0$ be any number such that $0 < \epsilon_0<g_i - \epsilon$, 
        \begin{align}
            \bbE\left[\sum_{n=1}^{\infty}\mathds{1}\left[\underline{g}_{n,i} \leq \epsilon \right]\right] &\leq t_i(\epsilon_0) + \frac{2M\sigma^2}{\epsilon_0^2}.  \label{222222}
        \end{align}
\end{enumerate}
\end{lemma}
\begin{proof}
    Here, we give a proof for~(\ref{111111}).
    \begin{align}
         \bbE\left[\sum_{n=1}^{\infty} \mathds{1}\left[\overline{g}_{n,i} > \epsilon \right] \right]& \leq t_i(\epsilon_0) + \sum_{n = t_i(\epsilon_0) + 1}^{\infty} \bbP\left( \overline{g}_{n,i} > \epsilon  \right)\nonumber\\
         & \leq t_i(\epsilon_0) + \sum_{n = t_i(\epsilon_0) + 1}^{\infty} M e^{-\frac{n\epsilon_{0}^2}{2\sigma^2}}\quad(Lemma~\ref{lemma2UpperBound})\nonumber\\
         & \leq t_i(\epsilon_0) + \sum_{n = 1}^{\infty}M e^{-\frac{n\epsilon_{0}^2}{2\sigma^2}}\nonumber\\
         & = t_i(\epsilon_0) + \frac{M}{e^{\frac{\epsilon_{0}^2}{2\sigma^2}} - 1}\quad\left(\sum_{t=1}^{\infty}e^{-ta} = \frac{1}{e^a - 1} \text{ for } a>0\right)\nonumber\\
         &\leq t_i(\epsilon_0) + \frac{2M\sigma^2}{\epsilon_0^2}\quad \left(e^a - 1 \geq a \text{ for }a>0\right)\nonumber.
    \end{align}
    The proof of~(\ref{222222}) is similar to the proof of~(\ref{111111}). 
    
\end{proof}

   
        


Let $\epsilon_0$ be such that $0<\epsilon_0 < \epsilon - g_i$ for any good arm $i$, and $$T_0 = K\max\limits_{i\in[K]}\lfloor t_i(\epsilon_0) \rfloor.$$

\begin{lemma}
\label{lemma4UpperBound}
    $\bbE\left[ \sum_{t=1}^{\infty} \mathds{1}\left[ i_t = i^* \right]\right] \leq t_{i^*}(\epsilon_0) + \frac{2M\sigma^2}{\epsilon_0^2} .$
\end{lemma}
\begin{proof}
    We have 
\begin{align}
\sum_{t=1}^{\infty} \mathds{1}\left[  i_t = i^*  \right] & = \sum_{t=1}^{\infty}\sum_{n=1}^{\infty} \mathds{1}[ i_t = i^*, T_{i^{*}}(t) = n] \label{equation16Upperbound}\\
&= \sum_{n=1}^{\infty} \mathds{1} \left[ \bigcup_{t=1}^{\infty}\{  i_t = i^*  , T_{i^{*}}(t) = n\}\right]\label{mediate5}\\
& \leq 1 + \sum_{n=2}^{\infty}\mathds{1} \left[ \overline{g}_{ n - 1, i^{*}} > \epsilon \right]\nonumber\\
& = 1 + \sum_{n=1}^{\infty}\mathds{1} \left[ \overline{g}_{n, i^{*},} > \epsilon \right],\nonumber
\end{align}
where step~(\ref{mediate5}) holds since if $i_t = i$~(arm $i$ is pulled at round $t$),  
 then $i$ has not been considered as a good arm and $\forall s \leq t-1, \overline{g}_{s,i}>\epsilon$.
By taking expectations on both sides, we have
\begin{align}
\bbE\left[ \sum_{t=1}^{\infty} \mathds{1}\left[ i_t = i^* \right]\right]
&\quad \leq  \bbE\left[ \sum_{n=1}^{\infty}\mathds{1} \left[ \overline{g}_{n, i^{*}} > \epsilon \right]\right]\nonumber\\
    &\quad\leq  t_{i^{*}}(\epsilon_0) + \frac{2M\sigma^2}{\epsilon_0^2}\quad\text{(Lemma~\ref{lemma3UpperBound}})\nonumber.
\end{align}
\end{proof}

\begin{lemma}
    \label{lemma4.5UpperBound}
    \[
        \bbE\left[ \sum_{t=1}^{T_0} \mathds{1}\left[i_t\neq i^*, {\underline{g}}_{t, i_t}\leq g_{i^*} + \epsilon_0 \right]\right] \leq \sum_{i\neq i^*}\frac{4\sigma^2\ln{\left(KMT_0\right)}}{{\left(g_i - g_{i^*} -2\epsilon_0\right)}^2}  + \frac{2(K-1)M\sigma^2}{{\epsilon_0}^2}.
    \]
\end{lemma}
\begin{proof}
    \begin{align*}
        &\sum_{t=1}^{T_0} \mathds{1}\left[i_t \neq i^*, {\underline{g}}_{t, i_t}\leq g_{i^*} + \epsilon_0\right]\\
        &\quad = \sum\limits_{i\neq i^{*}}\sum_{t=1}^{T_0}\sum_{n=1}^{T_0} \mathds{1}\left[ i_t = i, {\tilde{g}}_{t, i_t}\leq g_{i^*} + \epsilon_0, T_{i}(t) = n \right]\\
        &\quad = \sum\limits_{i\neq i^{*}}\sum_{n=1}^{T_0} \mathds{1}\left[ \bigcup_{t=1}^{T_0}\{ i_t = i, {\tilde{g}}_{t, i_t}\leq g_{i^*} + \epsilon_0 , T_{i}(t) = n \}\right]\\
        &\quad \leq \sum\limits_{i\neq i^{*}}\sum_{n=1}^{T_0}\mathds{1}\left[\hat g_{n,i} - \sqrt{\frac{4\sigma^2\ln{\left(KM{n}\right)}}{n}} 
        \leq g_{i^*} + \epsilon_0\right]\\
        &\quad = \sum\limits_{i\neq i^{*}}\sum_{n=1}^{T_0}\mathds{1}\left[\hat g_{n,i} - \sqrt{\frac{4\sigma^2\ln{\left(KM{n}\right)}}{n}} \leq g_{i} + (g_{i^*} - g_i) + \epsilon_0\right]\\
        &\quad \leq \sum\limits_{i\neq i^{*}}\sum_{n=1}^{T_0}\mathds{1}\left[\hat g_{n,i} - \sqrt{\frac{4\sigma^2\ln{\left(KMT_0\right)}}{n}} \leq g_{i} + (g_{i^*} - g_i) + \epsilon_0\right]\\
        &\quad \leq \sum\limits_{i\neq i^{*}}\left[\sum_{n=1}^{\frac{4\sigma^2\ln{\left(KMT_0\right)}}{{\left(g_i - g_{i^*} - 2\epsilon_0\right)}^2} }1 + \sum_{n = \frac{4\sigma^2\ln{\left(KMT_0\right)}}{{\left(g_i - g_{i^*} - 2\epsilon_0\right)}^2} }^{T_0} \mathds{1} \left[\hat g_{n,i} \leq g_i - \epsilon_0 \right]\right] 
    \end{align*}
    By taking expectations, 
    \begin{align}
       &\bbE\left[\sum_{t=1}^{T_0} \mathds{1}\left[i_t\neq i^*, \tilde {g}_{t}^{*}\leq g_{i^*} - 2\epsilon_0 \right]\right]\nonumber\\
       &\quad \leq \sum_{i\neq i^*} \frac{4\sigma^2\ln{\left(KMT_0\right)}}{{\left(g_i - g_{i^*} -2 \epsilon_0\right)}^2}  + \sum_{i\neq i^*}\sum_{n=1}^{\infty} \bbP\left( \hat g_{n,i} \leq g_i - \epsilon_0 \right)\nonumber\\
       &\quad \leq \sum_{i\neq i^*} \frac{4\sigma^2\ln{\left(KMT_0\right)}}{{\left(g_i - g_{i^*} -2\epsilon_0\right)}^2}  + \frac{(K-1)M}{e^{\frac{{\epsilon_0}^2}{2\sigma^2}} - 1}\quad(\text{Proposition~\ref{AppendixModifiedProposition1}})\nonumber\\
       &\quad\leq \sum_{i\neq i^*}\frac{4\sigma^2\ln{\left(KMT_0\right)}}{{\left(g_i - g_{i^*} -2\epsilon_0\right)}^2}  + \frac{2(K-1)M\sigma^2}{{\epsilon_0}^2}. \nonumber
    \end{align}
    
\end{proof}

\begin{lemma}
    \label{lemma5Upperbound}
    \begin{align*}
        \bbE\left[ \sum_{T_0 + 1}^{\infty} \mathds{1} [t\leq \stoppingtime] \right] &\leq \frac{K^2M}{2{\epsilon_0}^2}e^{4{\epsilon_0}^2}.
    \end{align*}
\end{lemma}
\begin{proof}
    In this case, some arms are pulled at least $\lceil (t-1) /K\rceil$ times until round $t$.
    \begin{align}
        &\bbE\left[\sum_{t=T_0 + 1}^{\infty} \mathds{1} \left[t\leq \stoppingtime] \right]\right]\nonumber\\
        &\quad \leq \bbE\left[\sum_{i=1}^{K}\sum_{t=T_0 + 1}^{\infty} \mathds{1}\left[ T_i(t)\geq \lceil (t-1) /K\rceil,t\leq \stoppingtime \right]\right]\nonumber\\
        &\quad \leq \bbE\left[ K\sum_{t=T_0 + 1}^{\infty} \mathds{1}\left[ T_i(t)\geq \lceil (t-1) /K\rceil,t\leq \stoppingtime \right]\right]\nonumber\\
        &\quad\leq \bbE\left[K\sum_{t=T_0 + 1}^{\infty} \mathds{1}\left[ \overline{g}_{i, \lceil (t-1) /K\rceil} > \epsilon \right] \right].\nonumber
    \end{align}
    Since $T_0 \geq t_i(\epsilon_0)$ for all $i\in[K]$, then 
    \begin{align}
        &\bbE\left[K\sum_{t=T_0 + 1}^{\infty} \mathds{1}\left[ \overline{g}_{i, \lceil (t-1) /K\rceil} > \epsilon \right] \right]\nonumber\\
        &\quad\leq KM\sum_{t = T_0 + 1}^{\infty} e^{-2{\epsilon_0}^2\left(  \lceil (t-1) /K \rceil  - 1\right)}\quad\text{(Lemma~\ref{lemma2UpperBound})}\nonumber\\
        &\quad\leq KM\sum_{t = T_0 + 1}^{\infty} e^{-2{\epsilon_0}^2\left( (t-1) /K - 1\right)}\nonumber\\
        &\quad\leq KM\int^{\infty}_{T_0} e^{-2{\epsilon_0}^2\left( (t-1) /K - 1\right)} dt \nonumber\\
        &\quad = KMe^{4{\epsilon_0}^2} \left[-\frac{K}{2{\epsilon_0}^2}e^{-2{\epsilon_0}^2t/K}\right]^{\infty}_{T_0}\nonumber\\
        &\quad = \frac{K^2M}{2{\epsilon_0}^2}e^{4{\epsilon_0}^2}e^{-2{\epsilon_0}^2T_0/K}\nonumber\\
        &\quad\leq \frac{K^2M}{2{\epsilon_0}^2}e^{4{\epsilon_0}^2}e^{- {2{\epsilon_0}^2}\max\limits_{i\in[K]}  \lfloor t_i(\epsilon_0) \rfloor}\nonumber\\
        &\quad \leq \frac{K^2M}{2{\epsilon_0}^2}e^{4{\epsilon_0}^2}.\nonumber
    \end{align}

\end{proof}

\begin{lemma}
    \label{lemma6UpperBound}
    \begin{equation*}
        \bbE\left[ \sum_{t=1}^{\infty} \mathds{1}[i_t\neq i^{*}, t\leq \stoppingtime, {\tilde{g}}_{t, i_t} > g_{i^{*}} + \epsilon_0 ] \right]\leq \frac{2\sigma^2 M}{{\epsilon_0}^2} + \frac{(K-1)K^2M}{2{\epsilon_0}^2}e^{4{\epsilon_0}^2}.
    \end{equation*}
\end{lemma}
\begin{proof}
    \begin{align}
        &\sum_{t=1}^{\infty}\mathds{1}\left[i_t \neq i^*, t\leq \stoppingtime, {\tilde{g}}_{t, i_t}  > g_{i^*} + \epsilon_0\right]\nonumber\\
        &\quad\leq \sum_{t=T_0 + 1}^{\infty}\mathds{1}\left[i_t \neq i^*, t\leq \stoppingtime\right] + \sum_{t=1}^{T_0}\mathds{1}\left[{\tilde{g}}_{t, i_t}  > g_{i^*} + \epsilon_0 \right] \label{Equalemma6UpperBound1}
    \end{align}
    Take expectation over the first part of~(\ref{Equalemma6UpperBound1}), 
    \begin{align}
        &\bbE\left[\sum_{t=T_0 + 1}^{\infty}\mathds{1}\left[i_t \neq i^*, t\leq \stoppingtime\right]\right] \nonumber\\
        &\quad\leq \sum_{t=T_0 + 1}^{\infty}\sum\limits_{i\neq i^*} \bbE\left[\mathds{1}\left[t\leq \stoppingtime\right]\right]~\quad(\text{Union bound}) \nonumber\\
        &\quad\leq \sum_{t=T_0 + 1}^{\infty} (K-1) \bbE\left[\mathds{1}\left[t\leq\stoppingtime\right]\right] \nonumber\\
        &\quad\leq \frac{(K-1)K^2M}{2{\epsilon_0}^2}e^{4{\epsilon_0}^2}~\quad(\text{Lemma~\ref{lemma5Upperbound}})\label{Equalemma6UpperBound2}.
    \end{align}
    As for the second part of~(\ref{Equalemma6UpperBound1}),
    \begin{align}
&\bbE\left[\sum_{t=1}^{T_0}\mathds{1}\left[{\tilde{g}}_{t, i_t}  > g_{i^*} + \epsilon_0 \right]\right] \nonumber\\
        &\quad = \sum_{t=1}^{T_0} \bbE\left[\mathds{1}\left[{\hat{g}}_{t, i_t} - \sqrt{\frac{2\sigma^2\ln{\left( KMT_i(t) \right)}}{T_i(t)}}  > g_{i^*} + \epsilon_0 \right]\right] \nonumber\\
        &\quad\leq \sum_{t=1}^{T_0} \bbE\left[ \mathds{1}\left[{\hat{g}}_{t, i_t} > g_{i^*} + \epsilon_0 \right] \right] \nonumber\\
        &\quad\leq\sum_{t=1}^{T_0}  M e^{-\frac{t{\epsilon_0}^2}{2\sigma^2}}\quad(\text{Proposition~\ref{prop:muLipschitz}})\nonumber\\
        &\quad\leq \frac{2\sigma^2 M}{{\epsilon_0}^2}\label{Equalemma6UpperBound3}.
    \end{align}
    Combining~(\ref{Equalemma6UpperBound2}) and~(\ref{Equalemma6UpperBound3}) leads to the result.

\end{proof}
\begin{proof}[Theorem~\ref{theoremUpperboundExpectationGoodArm}]
\begin{align*}
    \bbE[\stoppingtime] &= \bbE\left[\sum_{t=1}^{\infty}\mathds{1}[i_t = i^*,t\leq \stoppingtime] + \sum_{t=1}^{\infty} \mathds{1}[i_t \neq i^*, t\leq \stoppingtime]\right] \\
    &\leq \bbE\bigg[ \sum_{t=1}^{\infty}\mathds{1}[i_t=i^*] \\
    &\quad + \sum_{t=1}^{\infty}\mathds{1}[i_t\neq i^*, t\leq \stoppingtime, {\tilde{g}}_{t, i_t} \leq g_{i^* } + \epsilon_0]\\
    &\quad +\sum_{t=1}^{\infty} \mathds{1}[i_t\neq i^*, t\leq \stoppingtime, {\tilde{g}}_{t, i_t} > g_{i^*} + \epsilon_0] \bigg] \\
    &\leq \bbE\bigg[\sum_{t=1}^{\infty} \mathds{1}[ i_t = i^* ]\\
    &\quad + \sum_{t=1}^{T_0} \mathds{1}[i_t\neq i^* , {\tilde{g}}_{t, i_t} \leq g_{i^{*}} + \epsilon_0] \\
    &\quad + \sum_{T_0 + 1}^{\infty} \mathds{1} [t\leq \stoppingtime] \big]\\
    &\quad + \sum_{t=1}^{\infty} \mathds{1}[i_t\neq i^{*}, t\leq \stoppingtime, {\tilde{g}}_{t, i_t} > g_{i^{*}} + \epsilon_0 ] \bigg].
\end{align*}
 Then, the final result is obtained by combining lemma~\ref{lemma4UpperBound} to lemma~\ref{lemma6UpperBound}. 
\end{proof}

\subsection{Proof of Theorem~\ref{theoremUpperboundExpectationNogoodArm}}
\begin{proof}
Here, we let $\CalA\subseteq [K]$ be the active set of arms that have not been deleted.
\label{AppendixUpperBoundExpectationNogoodArm}
    \begin{align*}
        \stoppingtime = &\sum_{t=1}^{\infty}\mathds{1}\left\{ \text{Algorithm doesn't stop at trial }t\right\}\\
        &=\sum_{t=1}^{\infty}\mathds{1}\left[ \text{Neither conditions 1 nor 2 are satisfied at trial $t$} \right]\\
        &\leq \sum_{t=1}^{\infty}\mathds{1}\left[ \text{Condition 2 is not satisfied at trial $t$} \right]\\
        & = \sum_{t=1}^{\infty} \mathds{1}\left[ \text{Condition 2 is not satisfied at trial $t$} \right]\\
        & = \sum_{t=1}^{\infty} \mathds{1}\left[ \CalA \neq \emptyset\text{ at the end of trial $t$} \right]\\
        & \leq \sum_{i=1}^K\sum_{t=1}^{\infty}\mathds{1}\left[ i_t = i \text{ and } i \text{ is not deleted at trial $t$} \right]\\
        & = \sum_{i=1}^K\sum_{t=1}^{\infty}\mathds{1}\left[i_t = i\right]\mathds{1}\left[ \hat{g}_{t, i} \leq \alpha(T_{i}(t), \delta) \right]\\
        &=\sum_{i=1}^K\sum_{t=1}^{\infty}\sum_{n=1}^{\infty}\mathds{1}\left[i_t = i\right]\mathds{1}\left[ \hat{g}_{t, i} \leq \alpha(T_{i}(t), \delta) \right]\\
        &=\sum_{i=1}^K\sum_{n=1}^{\infty}\mathds{1}\left[ \underline{g}_{n,i} \leq 0 \right]\\
    \end{align*}
Taking the expectation, 
    \begin{align*}
        \bbE[\stoppingtime] &\leq\sum_{i=1}^K\sum\limits_{n=1}^{\infty}\mathbb{E}\left[ \mathds{1}\left[ \underline{g}_{n,i} \leq 0 \right]\right]\\
        &\leq \sum_{i\in[K]} t_i + \frac{KM\sigma^2}{\epsilon_0^2}\quad(\text{Lemma~\ref{lemma3UpperBound}}).
    \end{align*}
\end{proof}

\subsection{Proof of Theorem~\ref{ExpectationLowerbound}}
\label{AppendixLowerBoundExpectation}
\begin{definition}[Binary entropy function]
    $h:(0,1)\mapsto\mathbb{R}$ and $p\in(0,1)$,  
    \begin{align}
        h(p)&=-p\ln{p}-(1-p)\ln{(1-p)},
    \end{align}
    and $h(p)\leq\ln{2}$.
\end{definition}

\begin{proposition}
\label{PropositionLowerbound1}
    Let $\nu$ and $\overline{\nu}$ be two bandits models with $K$ arms and $M$ objectives such that for all $i\in[K]$ and $m\in[M]$, the distributions $\nu_i^{(m)}$ and $\overline{\nu}^{(m)}_i$ are mutually absolutely continuous. For any almost-surely finite stopping time $T$ and event $\mathcal{E}$, 
    \[
        \sum_{i = 1}^K \bbE[T_{i}(T)]\sum_{m=1}^{M}KL(\nu^{(m)}_i, \overline{\nu}_i^{(m)}) \geq d(\bbP_{\nu}[\mathcal{E}], \bbP_{\nu^{'}}[\mathcal{E}]), 
    \]
    where $KL(\nu_i, \nu_j)$ is the Kullback-Leibler divergence between distributions $\nu_i$ and $\nu_j$.
\end{proposition}
\begin{proof}
Denote the history until round $t$ as $\mathcal{F}_t = \{i_1, \ldots, i_t, \bmz_1, \ldots, \bmz_t\}$ and $\{ Y_{i,s}^{(m)}\}_{s=1,\ldots, T_i(t)}\sim\nu_i^{(m)}$ as the i.i.d. samples observed from arm $i$. 
Let $f^{(m)}_i$ be the density function for distribution $\nu_i^{(m)}$ with mean $\mu_i^{(m)}$ and define $\overline{f}^{(m)}_i$ similarly.

\begin{align}
    &\bbE\left[\ln\frac{\bbP_{\nu}(\mathcal{F}_t)}{\bbP_{\overline{\nu}}(\mathcal{F}_t)}\right] \nonumber\\
    &\quad = \bbE\left[ \sum_{i=1}^K\sum_{s=1}^t\mathds{1}(i_s = i)\sum_{m=1}^M \ln{\frac{f_i^{(m)}(\bmz_s^{(m)})}{\overline{f}_i^{(m)}(\bmz_s^{(m)})}}\right]\nonumber\\
    &\quad = \bbE\left[ \sum_{i=1}^K \sum_{s=1}^{T_i(t)}\sum_{m=1}^M \ln{\frac{f_i^{(m)}(Y_{i,s}^{(m)})}{\overline{f}_i^{(m)}(Y_{i,s}^{(m)})}} \right].\label{EquaLowerboundMediate2}
\end{align}
By definition, 
\[
\bbE_{\nu^{(m)}_i}\left[\ln{\frac{f_i^{(m)}(Y_{i,s}^{(m)})}{\overline{f}_i^{(m)}(Y_{i,s}^{(m)})}}\right] = KL(\nu_i^{(m)}, \overline{\nu}_i^{(m)}).
\]
Thus, with Wald's Lemma~\cite{wald2004sequential} we have  
\[
(\ref{EquaLowerboundMediate2}) = \sum_{i=1}^K\bbE\left[T_i(t)\right] \sum_{m=1}^M KL(\nu_i^{(m)}, \overline{\nu}_i^{(m)}).
\]
Then follows the analysis of~\cite{kaufmann2016complexitybestarmidentification} gives the conclusion.
\end{proof}
\begin{proof}[Theorem~\ref{ExpectationLowerbound}]
    Here we only give details on the case when a good arm exists, and the other case can be deduced in a similar way. 
    For any $\epsilon_1 > 0$ and fix a good arm $i^*$, consider a sequence of Bernoulli distributions associated with each arm $i\in[K]$ as $\{\overline{\nu}_i^{(1)}, \ldots, \overline{\nu}_i^{(M)}\}$ with expectations $\{ \overline{\mu}_i^{(1)}, \ldots, \overline{\mu}_i^{(M)}\}$ defined as 
    \begin{align*}
        \overline{\mu}_i^{(m)}& = 
            \left\{ 
            \begin{array}{cc}
                \xi_{m} - \epsilon_1, &\text{if } i = i^*, m = m^*_i\\
                {\mu}_i^{(m)}, &\text{otherwise},
            \end{array}
            \right.
    \end{align*}

where $m\in[M]$ and $m^*_i = \min\limits_{m\in[M]}{d\left(\mu_i^{(m)} , \xi_m\right)}$. 
Hence, $i^*$ is a good arm under $\bmmu$, and we denote the distributions with means $\{{\mu}_i^{(1)}, \ldots, {\mu}_i^{(M)}\}$ as $v_i := \{{\nu}_i^{(1)}, \ldots, {\nu}_i^{(M)}\}$ for convenience. 
Let $\mathcal{E}_{i^*} = \{\hat{a} = i^*\}$ and $p_{i^*} = \bbP\left(\hat{a} = i^* \right)$ under distribution $\nu_i$ with expectation $\bmmu_i$. 
From proposition~\ref{PropositionLowerbound1}, we obtain that
\begin{equation}\label{EquaLowerboundMediate}
    \bbE[T_{i^*}(T)]KL\left(\nu_{i^*}^{(m^*_{i^*})}, \overline{\nu}_{i^*}^{(m^*_{i^*})}\right)\geq d(p_{i^*}, \min\{\delta, p_{i^*}\}).
\end{equation}
The inequality~(\ref{EquaLowerboundMediate}) holds since $\mathcal{E}_{i^*}$ happens with probability $q_{i^*}\leq \delta$ under $\overline{\nu}_i$ and there is a case study as follows, 

\emph{Case 1: }For $\delta \leq p_{i^*}$, $d(p_{i^*},q) \geq d(p_{i^*}, \delta)$.

\emph{Case 2: }For $\delta > p_{i^*}$,  $d(p_{i^*}, q)\geq d(p_{i^*}, p_{i^*}) = 0$. 

Then, 
\begin{align*}
    d(p_{i^*}, \min\{\delta, p_{i^*}\}) &= \max\bigg\{ p_{i^*}\ln{\frac{1}{\min\{\delta, p_{i^*}\}}} - h(p) \\
    &\quad + (1-p_{i^*})\ln{\frac{1}{1- \min\{\delta, p_{i^*}\}}},0\bigg\}\quad(\text{Definition~\ref{defiBinaryRelativeentropy}}) \\
    &\geq\max\left\{p_{i^{*}}\ln{\frac{1}{\min\{\delta, p_{i^*}\}}} - \ln{2}, 0\right\}\quad(h(p) \leq \ln{2})\\
    &\geq\max\left\{p_{i^*}\ln{\frac{1}{\delta}} - \ln{2}, 0\right\}.
\end{align*}
Since $i^*$ is not a good arm under means $\left\{ \overline{\mu}_{i^*}^{(1)}, \ldots, \overline{\mu}_{i^*}^{(M)}\right\}$, for any good arm $i$, by definition of $(\delta, 0)$-success
\begin{align*}
    p_i & = \bbE_{\nu}[|[K]_{\text{good}} \cap \{\hat a\}|]\\
    & \geq \bbP_{\nu}[\{ \hat a \} \subseteq [K]_{\text{good}}]\\
    & \geq 1-\delta, 
\end{align*}
where $[K]_{\text{good}}\subseteq [K]$ is the set of good arms. Thus we obtain 
\[\sum_{i\in[K]}\bbE[T_i(T)] \geq  \max_{i\in{[K]_{\text{good}}}} \bbE[T_i(T)] \geq P^*
\]
and solve the lower bound as a optimization problem $P^*$
\begin{align*}
    &\text{minimize}\quad \frac{1}{ \max\limits_{i\in[K]_{\text{good}}}d\left({\mu^{(m^*_{i})}_{i}, \xi_{m^*_{i}}-\epsilon_1}\right)
    }\max\left\{ p_{i}\ln{\frac{1}{\delta} } - \ln{2}, 0\right\}, \\
    &\text{subject to}\quad  p_{i}  \geq 1 - \delta, \\
    &\qquad\qquad\quad\text{ } 0\leq p_{i} \leq 1 ,
\end{align*}
and reorganize the problem into 
\begin{align*}
    &\text{minimize}\quad \frac{q}{\max\limits_{i\in[K]_{\text{good}}} d\left({\mu^{(m^*_{i})}_{i}, \xi_{m^*_{i}}-\epsilon_1}\right)}, \\
    &\text{subject to}\quad p_{i} \geq 1 - \delta, \\
    &\qquad\qquad\quad\text{ } q \geq p_{i}\ln{\frac{1}{\delta}} - \ln{2},\\
    &\qquad\qquad\quad\text{ } q \geq 0, \\
    &\qquad\qquad\quad\text{ } 0\leq p_{i} \leq 1.
\end{align*}
Then, consider the dual problem 
\begin{align*}
    &\text{maximize}\quad (1-\delta)A - (\ln{2})B - C \\
    &\text{subject to}\quad B \leq \frac{1}{\max\limits_{i\in[K]_{\text{good}}} d\left({\mu^{(m^*_{i})}_{i}, \xi_{m^*_{i}}-\epsilon_1}\right)
    },\\
    &\qquad\qquad\quad\text{ } A - B \ln{\frac{1}{\delta}} - C \leq 0,\\
    &\qquad\qquad\quad\text{ } A, B, C\geq 0.
\end{align*}
Herein, we have the feasible solution given by
\begin{align*}
    A & = \frac{1}{\max\limits_{i\in[K]_{\text{good}}}d\left({\mu^{(m^*_{i})}_{i} , \xi_{m^*_{i}}-\epsilon_1}\right)}\ln{\frac{1}{\delta}}\\
    B & = \frac{1}{\max\limits_{i\in[K]_{\text{good}}}d\left({\mu^{(m^*_{i})}_{i} , \xi_{m^*_{i}}-\epsilon_1}\right)}\\
    C & = 0.
\end{align*}
Then put the results back into the objective function
\begin{align*}
    &(1-\delta)A - (\ln{2})B - C \\
    &\quad = (1-\delta){\frac{1}{\max\limits_{i\in[K]_{\text{good}}}d\left({\mu^{(m^*_{i})}_{i}, \xi_{m^*_{i}}-\epsilon_1}\right)}}\ln{\frac{1}{\delta}} - (\ln{2})\frac{1}{\max\limits_{i\in[K]_{\text{good}}}d\left({\mu^{(m^*_{i})}_{i}, \xi_{m^*_{i}}-\epsilon_1}\right)}\\
    &\quad = {\frac{1}{\max\limits_{i\in[K]_{\text{good}}}d\left({\mu^{(m^*_{i})}_{i}, \xi_{m^*_{i}}-\epsilon_1}\right)}} \ln{\frac{1}{2\delta}} - \frac{\delta}{\max\limits_{i\in[K]_{\text{good}}}d\left({\mu^{(m^*_{i})}_{i}, \xi_{m^*_{i}}-\epsilon_1}\right)}
\end{align*}
Since the existence of the duality gap, the result of the transformed problem is smaller than the original one, 
\begin{align*}
    \bbE[T_{\text{stop}}] & \geq {\frac{1}{\max\limits_{i\in[K]_{\text{good}}}d\left({\mu^{(m^*_{i})}_{i}, \xi_{m^*_{i}}-\epsilon_1}\right)}}\ln{\frac{1}{2\delta}} - \frac{\delta}{\max\limits_{i\in[K]_{\text{good}}}d\left({\mu^{(m^*_{i})}_{i}, \xi_{m^*_{i}}-\epsilon_1}\right)}\\
    & = {\frac{1}{\max\limits_{i\in[K]_{\text{good}}}d\left({\mu^{(m)}_{i}, \xi_{m}-\epsilon_1}\right)}}\ln{\frac{1}{2\delta}} - \frac{\delta}{\max\limits_{i\in[K]_{\text{good}}}d\left({\mu^{(m)}_{i}, \xi_{m}-\epsilon_1}\right)}.
\end{align*}
The final result is obtained by letting $\epsilon_1 \rightarrow 0$.
\end{proof}

\section{Supplementary of Experiments}
\label{AppendixExperiments}
\subsection{Additional Results}
Here, we provide other results not included in the main content. The error rate is the ratio of repetitions where the algorithm fails to output a good arm, and the symbol "–" indicates that the algorithm's error rate is higher than $50\%$. 
The exact results on stopping times are presented in Table~\ref{tab:experimentStoppingtimeDeltaSynthetic},~\ref{tab:experimentStoppingtimeEpsilonSynthetic},~\ref{tab:experimentStoppingtimeDeltaMedical} and~\ref{tab:experimentStoppingtimeEpsilonMedical}. The error rates for medical data presented in Section~\ref{section:experiments} are shown in Table~\ref{tab:experimentStoppingtimeDeltaMedical} 
 and~\ref{tab:experimentStoppingtimeEpsilonMedical}, and the standard deviations are in~\ref{tab:experimentStandard DeviationDeltaMedical} and~\ref{tab:experimentStandard DeviationEpsilonMedical}. 
\begin{table}[htbp]
    \centering
    \caption{Stopping Time w.r.t. $\delta$ with Synthetic Data}
     \label{tab:experimentStoppingtimeDeltaSynthetic}
    \begin{tabular}{c|c|c|c|c}
    \toprule
    $\delta$ & MultiAPT & MultiHDoC & MultiLUCB & \textbf{MultiTUCB} \\
    \midrule
    $0.005$ & $67558.19$ & $44837.80$ & $62850.43$ & \bm{$37889.20$} \\
    $0.010$ & $65445.56$ & $43764.69$ & $58366.11$ & \bm{$37184.51$} \\
    $0.015$ & $64890.75$ & $43106.50$ & $56017.03$ & \bm{$36712.40$} \\
    $0.020$ & $64149.97$ & $42792.54$ & $53833.43$ & \bm{$36359.47$} \\
    $0.025$ & $63211.24$ & $42123.42$ & $52583.60$ & \bm{$36104.91$} \\
    $0.030$ & $62911.17$ & $42188.60$ & $50985.18$ & \bm{$35887.75$} \\
    $0.035$ & $62387.45$ & $41883.71$ & $ 50186.66$ & \bm{$35701.17$} \\
    $0.040$ & $61927.60$ & $41853.96$ & $49157.64$ & \bm{$ 35557.45$} \\
    $0.045$ & $ 61919.50$ & $41515.10$ & $48519.28$ & \bm{$35399.62$} \\
    $0.050$ & $ 61754.00$ & $41357.94$ & $47518.73$ & \bm{$35263.71$} \\
    \bottomrule
    \end{tabular}
\end{table}

\begin{table}[htbp]
    \centering
    \caption{Stopping Time w.r.t. $\epsilon$ with Synthetic Data}
   \label{tab:experimentStoppingtimeEpsilonSynthetic}
    \begin{tabular}{c|c|c|c|c}
    \toprule
    $\epsilon$ & MultiAPT & MultiHDoC & MultiLUCB & \textbf{MultiTUCB} \\
    \midrule
    $0.002$ & $72142.25$ & $9351.76 $ & $68861.26$ & \bm{$42323.90$} \\
    $0.004$ & $68899.35 $ & $46246.50 $ & $ 64704.05 $ & \bm{$39260.29$} \\
    $0.006$ & $66254.14 $ & $43463.51 $ & $60994.38 $ & \bm{$36619.35$} \\
    $0.008$ & $63817.75 $ & $41095.90 $ & $57423.84 $ & \bm{$34091.37$} \\
    $0.010$ & $61281.93 $ & $38898.29 $ & $54010.93  $ & \bm{$31832.58$} \\
    $0.012$ & $59506.56 $ & $37015.99 $ & $51001.18 $ & \bm{$29879.56$} \\
    $0.014$ & $57565.94 $ & $35212.57 $ & $48344.01$ & \bm{$28139.68 $} \\
    $0.016$ & $55942.74 $ & $33555.54 $ & $ 45735.95 $ & \bm{$26371.34$} \\
    $0.018$ & $54355.53 $ & $32076.11 $ & $43349.82 $ & \bm{$24910.16$} \\
    $0.020$ & $53031.22$ & $30663.39$ & $41194.27$ & \bm{$23509.48$} \\
    \bottomrule
    \end{tabular}
\end{table}

\begin{table}[ht]
    \centering
    \caption{Stopping Time w.r.t. $\delta$ with Medical Data}
    \label{tab:experimentStoppingtimeDeltaMedical}
    \begin{tabular}{c|c|c|c|c}
    \toprule
    $\delta$ & MultiAPT & MultiHDoC & MultiLUCB & \textbf{MultiTUCB} \\
    \midrule
    $0.005$ & $159612.14 $ & $1930.59 $ & $2950.43 $ & \bm{$1512.23$} \\
    $0.010$ & $151168.96$ & $1881.55 $ & $2664.67 $ & \bm{$ 1472.03 $} \\
    $0.015$ & $ 145659.37 $ & $1852.99 $ & $ 2516.41$ & \bm{$1450.23 $} \\
    $0.020$ & $142957.35 $ & $1834.56 $ & $ 2406.29 $ & \bm{$1433.91$} \\
    $0.025$ & $140266.20 $ & $1821.28  $ & $2316.79  $ & \bm{$ 1422.12 $} \\
    $0.030$ & $138247.26$ & $1808.38 $ & $2221.46  $ & \bm{$ 1410.64$} \\
    $0.035$ & $ 136898.94$ & $ 1799.66$ & $ 2154.06 $ & \bm{$ 1399.08$} \\
    $0.040$ & $ 135402.13$ & $ 1792.78 $ & $2079.05$ & \bm{$  1391.86$} \\
    $0.045$ & $ 133913.38$ & $1785.77$ & $ 2039.52 $ & \bm{$ 1385.56 $} \\
    $0.050$ & $ 132941.90$ & $ 1777.64$ & $2006.03$ & \bm{$1379.53$} \\
    \bottomrule
    \end{tabular}
\end{table}

\begin{table}[ht]
    \centering
    \caption{Stopping Time w.r.t. $\epsilon$ with Medical Data}
    \label{tab:experimentStoppingtimeEpsilonMedical}
    \begin{tabular}{c|c|c|c|c}
    \toprule
    $\epsilon$ & MultiAPT & MultiHDoC & MultiLUCB & \textbf{MultiTUCB} \\
    \midrule
    $0.002$ & $-$ & $2041.31$ & $3004.62$ & \bm{$1625.77 $} \\
    $0.004$ &$173761.65$& $ 2025.90 $ & $ 2960.26$ & \bm{$ 1594.22$} \\
    $0.006$ & $137467.38$& $ 1986.86 $ & $ 2919.35  $  & \bm{$ 1545.89 $} \\
    $0.008$ & $ 108815.09 $  & $1933.00  $ & $2876.65 $  & \bm{$1505.27 $} \\
    $0.010$ & $ 90650.03  $  & $1884.11  $ & $ 2836.59  $  & \bm{$1462.63 $} \\
    $0.012$ & $ 76350.56 $  & $1828.88  $ & $2796.23$  & \bm{$ 1426.91 $} \\
    $0.014$ & $65814.72 $  & $1778.66 $ & $ 2755.69 $  & \bm{$1392.24$} \\
    $0.016$ & $57433.42$& $1729.29$ & $2716.41 $  & \bm{$ 1354.44$} \\
    $0.018$ & $50484.03 $  & $1683.60$ & $ 2675.94$  & \bm{$1327.58$} \\
    $0.020$ & $45030.93$  & $ 1639.93 $ & $ 2636.90$  & \bm{$1300.06$} \\
    \bottomrule
    \end{tabular}
\end{table}

\begin{table}[htbp]
    \centering
    \caption{Error Rate~(\%) w.r.t. $\delta$ with Medical Data}
    \label{tab:experimentErrorrateDeltaMedical}
    \begin{tabular}{c|c|c|c|c}
    \toprule
    $\delta$ & MultiAPT & MultiHDoC & MultiLUCB & MultiTUCB \\
    \midrule
    $0.005$ & $10.40 $ & $0.00 $ & $0.00$& $0.00$\\
    $0.010$ & $8.50$ & $0.00 $ & $0.00$& $0.00$\\
    $0.015$ & $7.15$ & $0.00 $ & $0.00$& $0.00$ \\
    $0.020$ & $6.60 $ & $0.00 $ & $0.00$& $0.00$ \\
    $0.025$ & $5.90 $ & $0.00 $ & $0.00$& $0.00$ \\
    $0.030$ & $5.50$ & $0.00 $ & $0.00$& $0.00$ \\
    $0.035$ & $ 5.35$& $0.00 $ & $0.00$& $0.00$\\
    $0.040$ & $5.10 $& $0.00 $ & $0.00$& $0.00$\\
    $0.045$ & $4.75$& $0.00 $ & $0.00$& $0.00$\\
    $0.050$ & $4.60$& $0.00 $ & $0.00$& $0.00$\\
    \bottomrule
    \end{tabular}
\end{table}

\begin{table}[htbp]
    \centering
    \caption{Error Rate~(\%) w.r.t. $\epsilon$ with Medical Data}
    \label{tab:experimentErrorrateEpsilonMedical}
    \begin{tabular}{c|c|c|c|c}
    \toprule
    $\epsilon$ & MultiAPT & MultiHDoC & MultiLUCB & MultiTUCB \\
    \midrule
    $0.002$ & $54.35$ & $0.00 $ & $0.00$& $0.00$\\
    $0.004$ &$ 22.45$& $0.00 $ & $0.00$& $0.00$\\
    $0.006$ & $ 4.55$& $0.00 $ & $0.00$& $0.00$\\
    $0.008$ & $ 0.75 $ & $0.00 $ & $0.00$& $0.00$\\
    $0.010$ & $ 0.00  $ & $0.00 $ & $0.00$& $0.00$\\
    $0.012$ & $ 0.00 $ & $0.00 $ & $0.00$& $0.00$\\
    $0.014$ & $0.00 $& $0.00 $ & $0.00$& $0.00$\\
    $0.016$ & $0.00$& $0.00 $ & $0.00$& $0.00$\\
    $0.018$ & $0.00 $ & $0.00 $ & $0.00$& $0.00$\\
    $0.020$ & $0.00$ & $0.00 $ & $0.00$& $0.00$ \\
    \bottomrule
    \end{tabular}
\end{table}


\begin{table}[ht]
    \centering
    \caption{Standard Deviation w.r.t. $\delta$ with Medical Data}
    \label{tab:experimentStandard DeviationDeltaMedical}
    \begin{tabular}{c|c|c|c|c}
    \toprule
    $\delta$ & MultiAPT & MultiHDoC & MultiLUCB & MultiTUCB \\
    \midrule
    $0.005$ & $40630.09$ & $5825.35$ & $814.65 $ & $4378.73 $ \\
    $0.010$ & $ 39685.57$ & $ 5680.28 $ & $776.37  $ & $  4288.27$ \\
    $0.015$ & $ 39149.34$ & $5662.82 $ & $ 753.44 $ & $ 4277.31 $ \\
    $0.020$ & $ 38927.07 $ & $5606.34 $ & $736.62 $ & $4263.95$ \\
    $0.025$ & $38662.88$ & $ 5595.31$ & $ 719.47$ & $ 4260.32 $ \\
    $0.030$ & $ 38529.29 $ & $ 5590.93 $ & $ 712.07 $ & $   4254.68$\\
    $0.035$ & $38365.54 $ & $5585.32 $ & $713.06 $ & $ 4252.60 $ \\
    $0.040$ & $ 38071.74$ & $5583.01$ & $691.61 $ & $4242.46  $ \\
    $0.045$ & $ 38105.41$ & $ 5560.48 $ & $ 683.35 $ & $ 4241.99 $ \\
    $0.050$ & $38080.04$ & $ 5546.13 $ & $675.84$ & $4240.76 $\\
    \bottomrule
    \end{tabular}
\end{table}

\begin{table}[ht]
    \centering
    \caption{Standard Deviation w.r.t. $\epsilon$ with Medical Data}
    \label{tab:experimentStandard DeviationEpsilonMedical}
    \begin{tabular}{c|c|c|c|c}
    \toprule
    $\epsilon$ & MultiAPT & MultiHDoC & MultiLUCB & MultiTUCB \\
    \midrule
    $0.002$ & $-$ & $4201.73$ & $831.45 $ & $1522.96 $ \\
    $0.004$ &$58647.06$& $ 24170.82 $ & $815.45 $ & $ 1453.56 $ \\
    $0.006$ & $ 36608.01$& $ 5176.59$ & $ 807.75 $  & $ 4141.31 $ \\
    $0.008$ & $ 31141.33  $  & $4798.52$ & $800.43 $  & $3673.10$\\
    $0.010$ & $26136.43 $  & $ 4356.11 $ & $788.47 $  & $ 3080.87 $\\
    $0.012$ & $21819.32$  & $  3823.17 $ & $777.36 $  & $ 2591.56  $ \\
    $0.014$ & $ 18650.05$  & $3452.81 $ & $763.97 $  & $2289.38 $ \\
    $0.016$ & $ 16375.48 $& $3044.30  $ & $751.69  $  & $ 1925.62 $\\
    $0.018$ & $14093.24 $  & $ 2760.02  $ & $739.96$  & $1736.17 $ \\
    $0.020$ & $12723.03 $  & $ 2581.48 $ & $730.53$  & $1631.91$  \\
    \bottomrule
    \end{tabular}
\end{table}

\end{document}